\documentclass{article}

\usepackage{microtype}
\usepackage{graphicx}
\usepackage{subfigure}
\usepackage{subcaption}
\usepackage{booktabs} 

\usepackage{hyperref}

\usepackage[accepted]{icml2024}

\usepackage{amsmath}
\usepackage{amssymb}
\usepackage{mathtools}
\usepackage{amsthm}

\usepackage[capitalize,noabbrev]{cleveref}

\DeclareMathOperator*{\argmin}{arg\,min}

\theoremstyle{plain}
\newtheorem{theorem}{Theorem}[section]

\theoremstyle{definition}
\newtheorem{definition}[theorem]{Definition}

\theoremstyle{remark}

\newtheorem{example}[theorem]{Example}
\usepackage[textsize=tiny]{todonotes}

\usepackage{enumitem}
\setitemize{itemsep=1pt,topsep=1pt}

\usepackage{algorithm}
\usepackage{algorithmic}
\usepackage{wrapfig}
\usepackage{multicol}
\usepackage{bbm}
\usepackage{subcaption}
\usepackage[export]{adjustbox}
\usepackage{microtype}      
\usepackage{xcolor}         
\usepackage{mathtools}
\usepackage{amsfonts}       
\usepackage{nicefrac}       

\icmltitlerunning{Automated Discovery of Functional Actual Causes in Complex Environments}

\newcommand{\algoname}[0]{Joint Optimization for Actual Cause Inference}
\newcommand{\algoacronym}[0]{JACI}

\begin{document}

\twocolumn[
\icmltitle{Automated Discovery of Functional Actual Causes in Complex Environments}
\icmlsetsymbol{equal}{*}

\begin{icmlauthorlist}
\icmlauthor{Caleb Chuck}{equal,ut}
\icmlauthor{
Sankaran Vaidyanathan}{equal,uma}
\icmlauthor{Stephen Giguere}{ut}
\icmlauthor{Amy Zhang}{ut}
\icmlauthor{David Jensen}{uma}
\icmlauthor{Scott Niekum}{uma}
\end{icmlauthorlist}

\icmlaffiliation{ut}{Department of Computer Science, University of Texas at Austin, Austin, Texas, USA}
\icmlaffiliation{uma}{College of Information and Computer Sciences, University of Massachusetts Amherst, Amherst, Massachusetts, USA}

\icmlcorrespondingauthor{Caleb Chuck}{calebc@cs.utexas.edu}
\icmlcorrespondingauthor{Sankaran Vaidyanathan}{sankaranv@cs.umass.edu}

\icmlkeywords{Actual Causality, Causal Inference, Reinforcement Learning, Model-Based RL, Normality, Context-Specific Independence}

\vskip 0.3in
]

\printAffiliationsAndNotice{\icmlEqualContribution} 


\begin{abstract}

Reinforcement learning (RL) algorithms often struggle to learn policies that generalize to novel situations due to issues such as causal confusion, overfitting to irrelevant factors, and failure to isolate control of state factors. These issues stem from a common source: a failure to accurately identify and exploit state-specific causal relationships in the environment. While some prior works in RL aim to identify these relationships explicitly, they rely on informal domain-specific heuristics such as spatial and temporal proximity.
\textit{Actual causality} offers a principled and general framework for determining the causes of particular events. However, existing definitions of actual cause often attribute causality to a large number of events, even if many of them rarely influence the outcome. Prior work on actual causality proposes \textit{normality} as a solution to this problem, but its existing implementations are challenging to scale to complex and continuous-valued RL environments. This paper introduces \textit{functional actual cause (FAC)}, a framework that uses \textit{context-specific independencies} in the environment to restrict the set of actual causes. We additionally introduce \textit{\algoname{} (\algoacronym{})}, an algorithm that learns from observational data to infer functional actual causes. We demonstrate empirically that FAC agrees with known results on a suite of examples from the actual causality literature, and \algoacronym{} identifies actual causes with significantly higher accuracy than existing heuristic methods in a set of complex, continuous-valued environments. 

\end{abstract}


\section{Introduction}

Desirable behavior in Reinforcement Learning (RL) settings is often comprised of complex sequences of context-dependent actions. Robust generalization requires an operational understanding of how complex patterns of states and actions interact to produce predictable outcomes. While some lines of research, such as representation learning~\cite{bengio2013representation} and dynamics modeling~\cite{polydoros2017survey}, aim to implicitly capture this complexity, other recent work has aimed to explicitly discover and leverage causal structure~\cite{scholkopf2021toward}.

However, many of these methods encode general relationships (how actions affect state in general) without explicitly considering the specific context (how a specific action affected a specific state). As a result, these methods can often struggle to provide significant benefit~\cite{dasgupta2019causal}. Recent work in RL has identified this problem with a variety of terms such as local causality \cite{pitis2020counterfactual}, interactions \citep{yang2023learning}, controllability \citep{seitzer2021causal}, and counterfactuals. However, they rely on domain-specific heuristics such as locality, proximity, and assumptions about object geometry. These approaches perform poorly in domains where such assumptions do not hold.

\textit{Actual causality} \citep{pearl2000book, halpern2016actual}, which formally describes the causes of particular events in a specific observed context, offers a principled and general solution to this problem. However, existing approaches in actual causality are challenging to scale effectively to high-dimensional and continuous-valued environments. This is because they are often overly permissive; an event qualifies as an actual cause if there is at least one counterfactual scenario where a different outcome could have occurred. For example, consider a robot trying to push a block across a table in a room filled with other objects. Every single object would be considered an actual cause, since any of them could have been placed in front of the block and prevented it from moving. 


\begin{figure*}
\centering     
\includegraphics[width=0.9\textwidth]{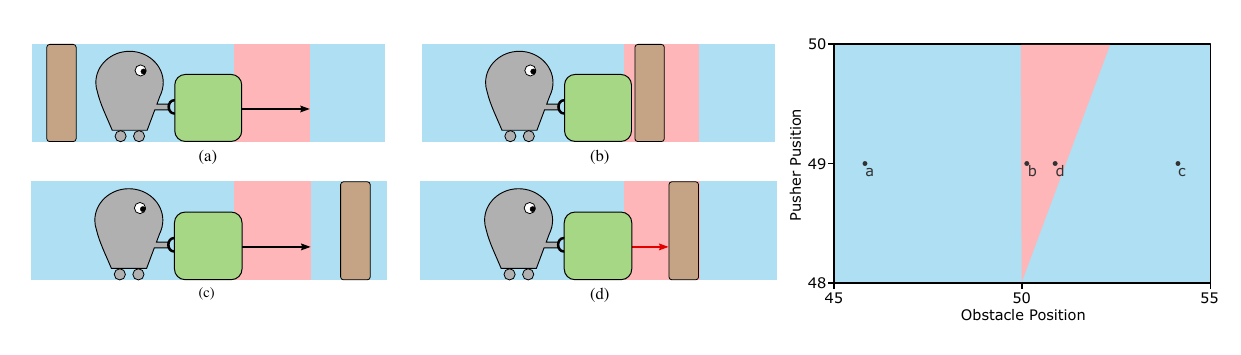}
\vspace{-0.5cm}
\caption{\textbf{(a)} \textbf{Invariant preimages} of the block position corresponding to a portion of the state space from Example \ref{block-pushing-example}. The block is at position 50 and the pusher attempts to move it to position 52. In states \textbf{(a)} and \textbf{(c)}, the obstacle has no impact on the block position. In state \textbf{(b)} the obstacle is directly in front of the block, and in state \textbf{(d)} it is one unit away and will obstruct the pusher. The blue region represents states where only the pusher can affect the block position---an \textit{invariant preimage} of the block position with respect to the pusher. When the observed state is in this region, the obstacle is not an actual cause of the block being pushed successfully. }
\label{fig:ivp-intro}
\vspace{-0.5cm}
\end{figure*}


This problem is broadly described in the actual causality literature as an issue of \textit{normality}~\cite{icard2017normality}, which describes the degree to which certain events should be considered as normal and others as rare. This has traditionally been represented using a rank ordering over states \cite{halpern2015graded}. However, manually specifying such an ordering is impractical in large and continuous domains with complex dynamics, and prior work does not provide guidance on how to obtain or learn this ordering. Overall, these approaches are infeasible in practical RL settings.

In this paper, we introduce \textit{functional actual cause (FAC)}, a framework that extends established definitions of actual cause by accounting for normality through \textit{context-specific independencies} that hold in different states. In the previous example, the robot was in a state where the final block position would have remained unaffected by the majority of rearrangements of the distant objects, even if they dramatically changed in position. To formalize this intuition, FAC learns sets of states where the outcome is partially independent of the variables whose values are not actual causes. These sets are referred to as \textit{invariant preimages (IVPs)} of the outcome with respect to the actual causes. By using the IVPs to prune out distracting events, FAC identifies a more focused set of actual causes that can be used for downstream tasks such as training and analyzing AI systems.

To learn IVPs and infer actual causes in high-dimensional environments, we introduce \textit{\algoname{} (\algoacronym{})}. \algoacronym{} learns from observational data to identify if the outcome at a given state is determined by only a specific subset of variables, and thereby infers functional actual causes of the outcome. We empirically verify that FAC maintains intuitive and agreed-upon verdicts in a set of examples from the actual causality literature. We additionally compare \algoacronym{} with local causality approaches in a set of complex and continuous-valued domains inspired by RL settings---\textit{Random Vectors}, \textit{Mini-Breakout}, and \textit{2D Robot Pushing}, and demonstrate significantly higher accuracy in recovering actual causes. Our approach generalizes local causality methods in the RL literature, and introduces actual causality to machine learning and RL settings.

\section{Illustrative Example}
\label{sec:example}
We use the following example, described in Figure \ref{fig:ivp-intro}, to illustrate how existing approaches fail to identify context-specific independencies and permit more actual causes. Appendix~\ref{addn_ex} includes formal analyses of additional examples.

\begin{example}[1-D Block Pushing]
\label{block-pushing-example}
    An agent, a block, and an obstacle are present in a 1-D world, and each occupies a position in a continuous range $[0, 100]$. The agent will always move 2 units to the right, and if the block is within this range, it will be pushed until the agent stops or it encounters an obstacle. Suppose the pusher is at position 49 with the block in front of it at position 50 and the obstacle behind it at position 20. Since the block is unobstructed, the pusher successfully pushes it forward to position 52.
\end{example}

In this example, existing definitions of actual cause would result in similar reasoning for both the pusher and the obstacle. The pusher would be an actual cause since the observed outcome (the block moving forward) would not have happened if the agent was in any other position. The obstacle position would also be an actual cause since the observed outcome would not have happened if the obstacle had been placed right in front of the block. This reasoning for the obstacle applies whether the obstacle is 30 meters away or 3 million meters away. If there are many such obstacles in the environment then \textit{all} of them will be actual causes. 

While these conclusions are technically correct, they can be impractical for downstream tasks in RL, since the field of view for the agent becomes prohibitively large. If everything in the environment is always an actual cause, it is difficult for an agent to maintain robustness and perform well. In Figure \ref{fig:ivp-intro}, the final block position only depends on the obstacle position when the current state is in the red region, and is independent of the obstacle position in all other states. We denote the red region as an \textit{invariant preimage} of the block position with respect to the the pusher position, where the outcome (block position) is invariant to all the remaining inputs (obstacle position in this case).

Similar ideas have been described in the actual causality literature as \textit{normality}, which accounts for when rare counterfactuals should be used to identify actual causes. A classic example of this is the \textit{Queen of England problem}: if one has a bed of flowers and neglects to water them, and the flowers withered, is the Queen of England an actual cause of the flowers' withering since she \textit{could} have watered them and did not. While this is technically correct, it is inconsistent with nearly all human judgments. The obstacle in Example \ref{block-pushing-example} is isomorphic to the Queen of England; the outcome is invariant to the obstacle position in a vast number of possible states even though the obstacle \textit{could} block the push, but it should not realistically be an actual cause in a context where could not actually block the push. The functional actual cause definition that we propose formalizes this intuition.

\section{Background and Related Work}

Our definition of actual cause builds on definitions introduced by Halpern and Pearl \citep{halpern2005causes, halpern2015modified}, as well as prior definitions of causal sufficiency and necessity \citep{beckers2021actual}. The definitions themselves are a subject of active and ongoing research, including extensions to modeling how humans make causal judgements~\citep{gerstenberg2021counterfactual}, first-order logic in temporal domains \citep{batusov2018situation}, and research on automatically verifying solutions to these definitions \citep{ibrahim2020checking}. We reference prior work on inferring actual causes and local causality in Section~\ref{sec:experiments}, and general applications of actual causality in Appendix~\ref{related-work}.

This paper is concerned with the problem of determining whether an event $\mathbf{X} = \mathbf{x}$ is an actual cause of an observed outcome $\mathbf{Y} = \mathbf{y}$ in a particular observed state, given an acyclic structural causal model $M$ with endogenous variables $\mathbf{V}$ and exogenous variables $\mathbf{U}$. We denote $\mathbf{S} = \mathbf{V} - \mathbf{Y}$ as the set of \textit{state variables}, which is the set of all endogenous variables excluding outcome variable $\mathbf{Y}$, and an assignment to these variables $\mathbf{S} = \mathbf{s}$ as a \textit{state}. Let $\mathcal{R}(\cdot)$ be a function that associates every variable or set of variables $\mathbf{X}$ with a nonempty set $\mathcal{R}(\mathbf{X})$ of possible values. 

For a given state $\mathbf s$, the assignment to the variables $\mathbf{X}$ is represented as $\mathbf{s_X}$. The function in $M$ that returns the value of $\mathbf{Y}$ is given by $\mathbf{Y} = f_\mathbf{Y}(\mathbf{s}, \mathbf{u})$, where in an abuse of notation we let $f_\mathbf{Y}$ take in a full state as an argument and assume it ignores all state variables that are descendants of $\mathbf{Y}$. The Boolean statement $(M,\mathbf{u}) \vDash [\mathbf{X} \leftarrow \mathbf{x}]\mathbf{Y} = \mathbf{y}$ is true when under the model $M$, the outcome $\mathbf{Y} = \mathbf{y}$ occurs whenever the exogenous variables take on values $\mathbf{u}$ and the intervention $\mathbf{X} \leftarrow \mathbf{x}$ is applied. 

Using this notation, we build a definition of actual cause from the following general framework. Existing definitions replace the following conditions with a specific interpretation of necessity, sufficiency, and minimality. Details on each of these conditions, including the witness set and prior definitions of actual causation, can be found in \citet{halpern2016actual} and \citet{beckers2021actual}.

\begin{definition}[General Framework for Actual Causation]
\label{general-ac-defn}

For a given model $M$ and observed state $\mathbf{s^*}$, the event $\mathbf{X} = \mathbf{x}$ is an \textit{actual cause} of the outcome $\mathbf{Y} = \mathbf{y}$ for a given context $\mathbf{U} = \mathbf{u}$ if:

\end{definition}

\begin{itemize}[leftmargin=1em]
    \item[] \label{ac1:general}\textbf{AC1} (\textit{Factual}): $\mathbf{X} = \mathbf{x}$ and outcome $\mathbf{Y} = \mathbf{y}$ actually happened: $\mathbf{s}_{\mathbf{X}}^* = \mathbf{x}$ and $(M, \mathbf{u}) \vDash (\mathbf{Y} = \mathbf{y})$.

    \item[] \label{ac2a:general}\textbf{AC2(a)} (\textit{Necessity}): There exists a counterfactual assignment $\mathbf{X} = \mathbf{x^\prime}$ and witness set $\mathbf{W}$ such that $(\mathbf{X} = \mathbf{x^\prime}, \mathbf{W} = \mathbf{s^*_W})$ is not sufficient for $\mathbf{Y} = \mathbf{y}$.

    \item[] \label{ac2b:general}\textbf{AC2(b)} (\textit{Sufficiency}): $(\mathbf{X} = \mathbf{x}, \mathbf{W} = \mathbf{s_W^*})$ is sufficient for $\mathbf{Y} = \mathbf{y}$.

    \item[] \label{ac3:general}\textbf{AC3} (\textit{Minimality}): There is no strict subset $\mathbf{Z} \subset \mathbf{X}$ for which the assignment $\mathbf{Z} = \mathbf{z}$ satisfies AC1, AC2(a) and AC2(b), where $\mathbf{z}$ refers to the values of variables $\mathbf{Z}$ in $\mathbf{x}$.
\end{itemize}

\section{Functional Actual Cause}
\label{sec:fac-theory}

In this section, we introduce a definition of actual cause that incorporates the intuitions of normality and context-specific independence described in Example \ref{block-pushing-example}. In the example, there are two partitions of the state space: one where only the pusher position is necessary to predict block position, and one where both the obstacle and block positions are necessary. We establish criteria for defining these partitions and use them to extend the definition of actual cause. 


\vspace{-0.3cm}
\subsection{Invariant Preimage}

We define an \textit{invariant preimage} $\mbox{\textit{IVP}}(\mathbf{X}; \mathbf{Y})$ to be a partition of the state space in which the outcome $\mathbf{Y}$ is invariant to all state variables that are not in $\mathbf{X}$. For any state $\mathbf{s}$ in this set, the outcome is not affected by any change to the values of the remaining variables $\mathbf{C} = \mathbf{S} - \mathbf{X}$ that results in another state $\mathbf{s}^\prime$ in the same partition.

\begin{definition}[Invariant Preimage]
    \label{defn-ivp-invariance}
    A partition of states $\mbox{\textit{IVP}}(\mathbf{X}; \mathbf{Y}) \subseteq \mathcal{R}(\mathbf{S})$ is an \textit{invariant preimage} of the outcome variables $\mathbf{Y}$ with respect to the variables $\mathbf{X}$ if all states $\mathbf{s}$ within the partition satisfy:
    \vspace{-0.3em}
    \begin{equation} \label{IVP definition}
        \begin{aligned}
            f_\mathbf{Y}(\mathbf{s_X}, \mathbf{s_C}, \mathbf{u}) = f_\mathbf{Y}(\mathbf{s_X}, \mathbf{s^\prime_C}, \mathbf{u}) \\ \forall \: \mathbf{s^\prime} \in \mbox{\textit{IVP}}(\mathbf{X}; \mathbf{Y}) \text{ where } \mathbf{s^\prime_X} = \mathbf{s_X}
        \end{aligned}
    \end{equation}
\end{definition}
\vspace{-0.3em}
Each invariant preimage signifies a partition of the state space in which a particular context-specific independence holds. For example, if a state $\mathbf{s^*}$ was present in the set of states given by $\mbox{\textit{IVP}}(\mathbf{X}; \mathbf{Y})$, then the value of the outcome $\mathbf{Y}$ remains constant for every other state in the set where $\mathbf{X}$ takes on the same values $\mathbf{s^*_X}$. We refer to a set of invariant preimages with respect to outcome variable $\mathbf{Y}$ as $\mathcal{I}(\mathbf{Y}) = \{\mathcal{I}_1, \mathcal{I}_2, ..., \mathcal{I}_K \}$ where $K = 2^{|\mathbf{S}|}$, and each set in $\mathcal{I}(\mathbf{Y})$ corresponds to a subset of the state variables $\mathbf{S}$. 

Figure \ref{fig:ivp-intro} describes a 2D state space with two invariant preimages. The triangle-shaped region in the center corresponds to $\mbox{\textit{IVP}}((\mathbf{P}, \mathbf{O}); \mathbf{Y})$ where the block position $\mathbf{Y}$ depends on the pusher position $\mathbf{P}$ and obstacle position $\mathbf{O}$. The large remaining space corresponds to $\mbox{\textit{IVP}}(\mathbf{P}; \mathbf{Y})$, where $\mathbf{Y}$ is unaffected by obstacle position $\mathbf{O}$. In Section \ref{sec:algorithms}, we devise discovery algorithms to obtain a set of invariant preimages $\mathcal{I}(\mathbf{Y})$ and a mapping from states to these invariant preimages. For the rest of this section, we use a given $\mathcal{I}(\mathbf{Y})$ to reason about causal necessity, sufficiency, and minimality. 

\subsection{Causal Sufficiency}
\label{sec-sufficiency}

In the language of causal models, causal sufficiency is established by treating $\mathbf{X} = \mathbf{x}$ as an intervention and $\mathbf{Y} = \mathbf{y}$ as the consequence of that intervention. However, this does not specify what values the remaining state variables $\mathbf{C}$ should take, and different choices of conditions for $\mathbf{C}$ lead to different definitions of causal sufficiency. If we choose not to intervene on any other state variables and just let them take on their observational values, we would obtain the following definition, reproduced from \citet{beckers2021actual}:

\begin{definition}[Weak Sufficiency]
\label{defn-weak-sufficiency-witness}
    For an observed state $\mathbf{s}^*$, the event $\mathbf{X} = \mathbf{x}$ is weakly sufficient for the outcome $\mathbf{Y} = \mathbf{y}$ with witness set $\mathbf{W}$ if $(M, \mathbf{u}) \vDash [\mathbf{X} \leftarrow \mathbf{x}, \mathbf{W} \leftarrow \mathbf{w}^*] \mathbf{Y} = \mathbf{y}$.
\end{definition}

This definition is referred to as \textit{weak} because it still allows for $\mathbf{Y} = \mathbf{y}$ to be affected by changes in the values of the remaining variables $\mathbf{C}$, even if they are not intervened on. This ignores the fact that in a particular context, the values of some of these variables may have no bearing on the observed outcome. In other words, there is a \textit{context-specific independence} between the outcome $\mathbf{Y}$ and the remaining variables $\mathbf{C}$ in this particular state. 

The invariant preimages precisely define this context-specific independence. If the observed state $\mathbf{s}^*$ fell within the set $\mbox{\textit{IVP}}(\mathbf{X}; \mathbf{Y})$ then we would know that the outcome value depended solely on the variables $\mathbf{X}$. To ensure that this independence holds, we extend the weak sufficiency definition by additionally requiring that the observed state fall within the invariant preimage $\mbox{\textit{IVP}}(\mathbf{X}; \mathbf{Y})$ for the event $\mathbf X =\mathbf x$ to be deemed as causally sufficient.

\begin{definition}[Functional Sufficiency] 
\label{defn-functional-sufficiency}
    For an observed state $\mathbf{s}^*$ and set of invariant preimages $\mathcal{I}_\mathbf{Y}$, an event $\mathbf{X} = \mathbf{x}$ is functionally sufficient for the outcome $\mathbf{Y} = \mathbf{y}$ if it is weakly sufficient for $\mathbf{Y} = \mathbf{y}$ and $\mathbf{s}^* \in \mbox{\textit{IVP}}(\mathbf{X}; \mathbf{Y})$.
\end{definition}
\vspace{-0.3em}

The invariant preimages are therefore used to associate each actual cause with a corresponding context-specific independence. We note that the definition does not change regardless of the witness set used to satisfy weak sufficiency, since holding a subset of the observed state $\mathbf{s}^*$ fixed does not change which invariant preimage it belongs to.

\subsection{Causal Necessity}

Necessity is typically treated as a violation of sufficiency in the absence of the candidate actual cause. In the language of SCMs, this is given by the existence of at least one alternative event $\mathbf{X} = \mathbf{x^\prime}$ that can result in a different outcome $\mathbf{Y} \neq \mathbf{y}$, thereby not being sufficient for the actual observed outcome $\mathbf{Y} = \mathbf{y}$. 

\vspace{-0.1em}
\begin{definition}[Contrastive Necessity]
    \label{defn-necessity}
    For an observed state $\mathbf{s}^*$, the event $\mathbf{X} = \mathbf{x}$ is \textit{necessary} for the outcome $\mathbf{Y} = \mathbf{y}$ under the witness set $\mathbf{W}$ if there exists $\mathbf{x}^\prime \in \mathcal{R}(\mathbf{X})$ such that $(M, \mathbf{u}) \vDash [\mathbf{X} \leftarrow \mathbf{x}^\prime, \mathbf{W} \leftarrow \mathbf{s^*_W}] \mathbf{Y} \neq \mathbf{y}$.
\end{definition}
\vspace{-0.1em}

In practice, this condition can be extended such that an event is declared as necessary only if a change in $\mathbf{X}$ results in a different outcome in at least $\alpha$ percentage of states, instead of simply verifying that one such state exists. We refer to this as $\alpha$-necessity and its corresponding sufficiency condition as $\alpha$-sufficency. We expand on their definitions and properties in Appendix \ref{appendix: alpha-splitting}, and this is a direction for future work, but for the rest of this section we use weak sufficiency and contrastive necessity to outline our approach.


\subsection{Minimality}

In the framework for actual causation given in Definition \ref{general-ac-defn}, \hyperref[ac3:general]{AC3} enforces minimality and the elimination of redundant causes. This is traditionally viewed as a local property of the observed state. On the other hand, minimizing the number of variables in a functional actual cause also implies minimizing the number of variables corresponding to the IVP that contains the observed state. However, \hyperref[ac3:general]{AC3} must hold for any observed state where we want to find actual causes, and this depends on how states are assigned to IVPs. Hence minimality is a global property of the set of IVPs. 

To formalize global minimality, we introduce the following notation. Each invariant preimage in $\mathcal{I}(\mathbf{Y})$ is associated with a unique binary vector of length $\mathbf{|S|}$, where each element corresponds to a variable in $\mathbf{S}$. The binary vector corresponding to $\mbox{\textit{IVP}}(\mathbf{X}; \mathbf{Y})$ will have a value one in the indices corresponding to variables in $\mathbf{X}$. We refer to $(\mathcal{B}(\mathbf{Y}), \mathcal{I}(\mathbf{Y}))$ as a binary-subset pair. Minimality can thereafter be represented using a cost function as follows:

\begin{definition}[Minimality of binary-subset pair]
    \label{defn-ivp-minimality}
    A set of invariant preimages $\mathcal{I}(\mathbf{Y}) = \{\mathcal{I}_1, \mathcal{I}_2, ..., \mathcal{I}_K \}$ and corresponding binary vectors $\mathcal{B}(\mathbf{Y}) = \{\mathbf{b}_1, \mathbf{b}_2, ..., \mathbf{b}_K\}$ are minimal if:
    \vspace{-0.3em}
    \begin{equation} \label{IVP minimality}
        \begin{aligned}
            (\mathcal{B}(\mathbf{Y}), \mathcal{I}(\mathbf{Y})) = \argmin_{\mathcal{B}(\mathbf{Y}), \mathcal{I}(\mathbf{Y})} \: \sum_{k=1}^K|\mathcal{I}_k||\mathbf{b}_k|
        \end{aligned}
    \end{equation}
    
\end{definition}
\vspace{-0.3em}
We can recover AC3 by showing that if the binary-subset pair $(\mathcal{B}(\mathbf{Y}), \mathcal{I}(\mathbf{Y}))$ is minimal according to Equation \ref{IVP minimality}, then any event $\mathbf{X} = \mathbf{x}$ that satisfies AC1 and Definitions \ref{defn-functional-sufficiency} and \ref{defn-necessity} for the outcome $\mathbf{Y} = \mathbf{y}$ must also be minimal.

\begin{theorem}
\label{minimality-proof}
    For an observed state $\mathbf{s^*}$ and binary-subset pair $(\mathcal{B}(\mathbf{Y}), \mathcal{I}(\mathbf{Y}))$, if the event $\mathbf{X} = \mathbf{x}$ satisfies AC1, functional sufficiency, and contrastive necessity for the outcome $\mathbf{Y} = \mathbf{y}$ under witness set $\mathbf{W}$, and $(\mathcal{B}(\mathbf{Y}), \mathcal{I}(\mathbf{Y}))$ is minimal, then there cannot exist a smaller set of variables $\mathbf{Z} \subset \mathbf{X}$ such that $\mathbf{Z} = \mathbf{z}$ also satisfies AC1, functional sufficiency, and contrastive necessity for $\mathbf{Y} = \mathbf{y}$, where $\mathbf{z}$ denotes the values of variables $\mathbf{Z}$ in the assignment $\mathbf{x}$.
\end{theorem}


\subsection{Actual Causation}
\label{sec:fac}

By putting together the definitions of functional sufficiency (Definition \ref{defn-functional-sufficiency}), functional necessity (Definition \ref{defn-necessity}), and minimality (Definition \ref{defn-ivp-minimality}) with the general framework for actual causation (Definition \ref{general-ac-defn}), we obtain the following definition of actual cause.

\begin{definition}[Functional Actual Cause]
\label{fac-defn}

For a given model $M$, observed state $\mathbf{s^*}$, and binary-subset pair $(\mathcal{B}(\mathbf{Y}), \mathcal{I}(\mathbf{Y}))$, the event $\mathbf{X} = \mathbf{x}$ is a functional actual cause of the outcome $\mathbf{Y} = \mathbf{y}$ for a given context $\mathbf{U} = \mathbf{u}$ if:

\end{definition}

\vspace{-0.2cm}
\begin{itemize}[leftmargin=1em]
    \item[] \label{ac1:fac}\textbf{AC1} (\textit{Factual}): $\mathbf{X} = \mathbf{x}$ and outcome $\mathbf{Y} = \mathbf{y}$ actually happened, i.e. $\mathbf{s}_{\mathbf{X}}^* = \mathbf{x}$ and $\mathbf{s}_{\mathbf{Y}}^* = \mathbf{y}$.

    \item[] \label{ac2a:fac}\textbf{AC2(a)} (\textit{Necessity}): There exists a counterfactual event $\mathbf{x}^\prime \in \mathcal{R}(\mathbf{X})$ and witness set $\mathbf{W}$ such that $(M, \mathbf{u}) \vDash [\mathbf{X} \leftarrow \mathbf{x}^\prime, \mathbf{W} \leftarrow \mathbf{s^*_W}] \mathbf{Y} \neq \mathbf{y}$.

    \item[] \label{ac2b:fac}\textbf{AC2(b)} (\textit{Sufficiency}): $\mathbf{s}^* \in \mbox{\textit{IVP}}(\mathbf{X}; \mathbf{Y})$

    \item[] \label{ac3:fac}\textbf{AC3} (\textit{Minimality}): The binary-subset pair $(\mathcal{B}(\mathbf{Y}), \mathcal{I}(\mathbf{Y}))$ minimizes the cost function $\mathcal{L} = \sum_{k=1}^K|\mathcal{I}_k||\mathbf{b}_k|$
\end{itemize}

As with the Modified HP definition \cite{halpern2015modified}, weak sufficiency is dropped from the definition since it is trivially satisfied if \hyperref[ac1:fac]{AC1} is satisfied. Like other definitions of actual cause, functional actual causes are minimal sets of events that are necessary and sufficient for the observed outcome, but are additionally required to satisfy a globally minimized context-specific independence property. We show in the next section how these properties capture the normality intuition outlined in Section \ref{sec:example}.


\subsection{FAC and Normality}
In the literature on actual causality, \textit{normality} is used to formalize human judgments about normal events and rare events \cite{icard2017normality}. Normal and rare events often appear in machine learning problems as well, where it is commonly hypothesized that in large and high-dimensional spaces, many real-world datasets actually lie along a lower-dimensional manifold within the space \citep{cayton2005manifold}. To tractably infer actual causes in such spaces, and avoid Queen of England-style problems like in Example \ref{block-pushing-example}, it is essential to identify rare events and know when to exclude them.

To address the problem of normality, the FAC definition formalizes two key intuitions. Firstly, it ensures that events that could not have affected the outcome that was actually observed are not used to determine actual causation. This is achieved using invariant preimages, which enforce the context-specific independence property that holds in the observed state. Secondly, it ensures that when a rare event actually occurs, it is prioritized as an actual cause over an event that occurs routinely. This is particularly useful for RL, where distinguishing normal events with no effect (the obstacle being far away) from rare events with causal effects (the obstacle blocking the agent) is often key to robust control. This concept, known as \textit{abnormal inflation}~\citep{icard2017normality}, also underlies many prior works in philosophy and actual causation~\citep{halpern2015graded, blanchard2017normality}, and our approach preserves this intuition.

By enforcing the minimality criterion in Equation \ref{IVP minimality}, a binary subset pair $(\mathcal B(\mathbf Y), \mathcal I(\mathbf Y))$ has lower cost when an invariant preimage $\mbox{\textit{IVP}}(\mathbf{X}; \mathbf{Y})$ with more variables in $\mathbf X$ contains fewer states, so events with lower cardinality (smaller $|\mathbf X|$) are actual causes in more states. Assuming equal likelihood for all states, which is assumed in most prior work on actual causation, larger sets represent normal events and smaller sets represent rare events. This can be extended to probabilistic settings as well by weighting states by their likelihood, and we leave this to future work. Thus, through invariant preimages and global minimality, the FAC definition identifies events useful to a decision-making agent (e.g. the obstacle actually blocking the agent, or the Queen of England actually watering the flowers) without focusing on irrelevant ones (e.g. rearrangements of distant obstacles).


\subsection{Relationships between FAC and existing definitions}
\label{FACcomparison}

The most commonly used definition of actual causation is Modified HP \citep{halpern2015modified}, which is based on weak sufficiency. Functional sufficiency extends this definition with a normality constraint by requiring the observed state $\mathbf{s}^*$ to fall within the invariant preimage $\mbox{\textit{IVP}}(\mathbf{X}; \mathbf{Y})$, where the invariant preimages satisfy a global minimality property. If all states in the state space fell within the same invariant preimage $\mbox{\textit{IVP}}(\mathbf{S}; \mathbf{Y})$, then the constraint is trivially satisfied and functional sufficiency would be equivalent to weak sufficiency. We can interpret this as a property of the environment where no context-specific independencies exist and the outcome always depends on all state variables, or just a choice of the set of invariant preimages $\mathcal{I}(\mathbf{Y})$ that ignores context-specific independencies.

Beckers (\citeyear{beckers2021actual}) additionally defines \textit{direct sufficiency}, which imposes a stronger constraint by requiring that the outcome $\mathbf{Y} = \mathbf{y}$ always occurs when the intervention $\mathbf{X} \leftarrow \mathbf{x}$ is applied, regardless of the values of all other variables $\mathbf{C}$. This is analogous to saying that the observed state $\mathbf{s}^*$ lies within the invariant preimage $\mbox{\textit{IVP}}(\mathbf{X}; \mathbf{Y})$. However, while invariant preimages are defined such that the invariance to $\mathbf{C}$ only needs to hold for a subset of the possible values of $\mathbf{C}$, direct sufficiency requires that the invariance holds for all possible assignments of $\mathbf{C}$. 

\begin{theorem}
    \label{thm-fac-direct-sufficiency}
    For an observed state $\mathbf{s}^*$, functional sufficiency and direct sufficiency are equivalent when the invariant preimage $\mbox{\textit{IVP}}(\mathbf{X}; \mathbf{Y})$ that contains $\mathbf{s}^*$ shows complete invariance, i.e. all other states $\mathbf{s}^\prime$ where $\mathbf{s^\prime_X} = \mathbf{s^*_X}$ also lie within $\mbox{\textit{IVP}}(\mathbf{X}; \mathbf{Y})$.
\end{theorem}
\vspace{-0.6em}
While functional sufficiency is a statement about contextual invariance, direct sufficiency requires complete invariance and weak sufficiency does not require any invariance. Hence functional sufficiency captures the spectrum between weak sufficiency and direct sufficiency. Appendices \ref{appendix: direct-sufficiency} and \ref{addn_ex} expand further on the relationships between FAC and existing definitions of actual cause. 

\section{Discovering Actual Causes}
\label{sec:algorithms}

A key implication of the FAC definition (Definition \ref{fac-defn}) is that the process of inferring functional actual causes boils down to mapping each observed state $\mathbf{s}^*$ to a binary vector $\mathbf{b} \in \mathcal{B}(\mathbf{Y})$. 
To ensure that Definition \ref{fac-defn} is satisfied for the actual causes in all possible states, we verify that the binary-subset pair $(\mathcal{B}(\mathbf{Y}), \mathcal{I}(\mathbf{Y}))$ satisfies the following properties:

\begin{itemize}[leftmargin=0.5em]
\itemsep0em
    \item[] \label{property:invariance}\textbf{P1} (\textit{Invariance}): For any set of variables $\mathbf{X}$, all states in its corresponding partition $\mathcal{I}_\mathbf{X}$ satisfy Equation \ref{IVP definition}.
    \item[] \label{property:necessity}\textbf{P2} (\textit{Necessity}): If a state $\mathbf{s}$ lies within the partition $\mathcal{I}_\mathbf{X}$ then $\mathbf{X} = \mathbf{s_X}$ satisfies contrastive necessity (Definition \ref{defn-necessity}) with respect to the outcome at state $\mathbf{s}$. 
    \item[] \label{property:minimality}\textbf{P3} (\textit{Minimality}): $(\mathcal{B}(\mathbf{Y}), \mathcal{I}(\mathbf{Y}))$ satisfies Definition \ref{defn-ivp-minimality}.
\end{itemize}

To discover a binary-subset pair that satisfies these conditions, we start with a naive approach: exhaustively enumerating all possible partitionings of the state space and assignments of binary vectors to these partitions. We then identify and return all binary-subset pairs that satisfy properties \hyperref[property:invariance]{\textbf{P1}}, \hyperref[property:necessity]{\textbf{P2}}, and \hyperref[property:minimality]{\textbf{P3}}. This approach is described in Algorithm \ref{exhaustive-search}, with further details in Appendix \ref{appendix: algorithm-0}.

We evaluate our approach on a set of well-known examples from the actual causality literature and demonstrate FAC agrees with the majority of established results in Appendix \ref{appendix-exhaustive}. We also apply this algorithm to a simplified 1-D mover environment similar to Example \ref{block-pushing-example}. The results shown in Table \ref{fig:results:1dmover} indicate that the obstacle is declared as an actual cause only in states where it could have impeded the mover, matching the intuition of normality outlined in Example \ref{block-pushing-example}.

\setcounter{algorithm}{-1}
\begin{algorithm}
\caption{Exhaustive Search}
\begin{algorithmic}
\label{exhaustive-search}
  \STATE {\bfseries Input:} SCM $M$, range function $\mathcal{R}$, soft ratios $(\alpha_0, \alpha_1)$
  \STATE {\bfseries Output:} min-cost binary-subset partition $(\mathcal{B}(\mathbf{Y}), \mathcal{I}(\mathbf{Y}))$
    \STATE \textbf{Let} $g(\mathbf b, \mathcal I) \rightarrow \{0,1\}$ indicate if a binary $\mathbf b$ violates invariance, $\alpha-$necessity/sufficiency on state subset $\mathcal I$.
    \STATE \textbf{For} all binary-subset partitions $(\mathcal B(\mathbf Y), \mathcal I(\mathbf Y)) \in \mathcal B \times \mathcal P$. $\mathcal B=$ all binaries of $\{0,1\}^n$ and $\mathcal P=$ all partitions of $\mathcal S$.
    \STATE \quad \textbf{if} $\bigcap_{\mathbf b, \mathcal I \in (\mathcal B(\mathbf Y), \mathcal I(\mathbf Y))} g(\mathbf b, \mathcal I)$ and unique binaries
    \STATE \quad \quad \textbf{Add} $\left[(\mathcal B(\mathbf Y), \mathcal I(\mathbf Y))\rightarrow \sum_{j} |\mathbf{b}_j| |\mathcal{I}_j|\right ]$ to cost dictionary $\mathcal C$
    \STATE \textbf{Return} $\argmin_{(\mathcal B(\mathbf Y), \mathcal I(\mathbf Y))}\mathcal{C}$
\end{algorithmic}
\end{algorithm}

\begin{table}
\centering
\begin{tabular}[width=0.5\linewidth,valign=b]{|l|c|c|c||c|c|c|r|}
\toprule
$\mathbf{b}$& $m$& $o$& $m^\prime$& $\mathbf{b}$& $m$& $o$& $m^\prime$\\
\midrule
$11$& $0$& $1$& $0$ & $10$& $2$& $0$& $3$\\
$11$& $1$& $2$& $1$ & $10$& $2$& $1$& $3$\\
$10$& $0$& $0$& $1$ & $10$& $2$& $2$& $3$\\
$10$& $0$& $2$& $1$ & $10$& $3$& $0$& $4$\\
$10$& $1$& $0$& $2$ & $10$& $3$& $1$& $4$\\
$10$& $1$& $1$& $2$ & $10$& $3$& $2$& $4$ \\
\bottomrule
\end{tabular}
\caption{Table of binary, state and outcome from Algorithm \ref{exhaustive-search} in 1D mover environment. Abbreviations: $\mathbf{b}$: binary [mover, obstacle], $m$: mover, $o$: obstacle, $m^\prime$: next mover position. 1D mover dynamics: $m^\prime = m + 1$ if $o \neq m + 1$, otherwise $m^\prime = m$.}
\label{fig:results:1dmover}
\vspace{-0.7cm}
\end{table}

\subsection{Recovering Binary-Subset Pairs by Joint Optimization}

To devise algorithms that scale to more complex domains, we learn a parameterized function that maps states to binary vectors, represented using a neural network $h_\mathbf{Y}(\mathbf{s}; \theta): \mathcal{R}(\mathbf{S}) \to \{0,1\}^{|\mathbf{S}|}$. This exploits the fact that once we have a binary-subset pair that satisfies properties \hyperref[property:invariance]{\textbf{P1}}, \hyperref[property:necessity]{\textbf{P2}}, and \hyperref[property:minimality]{\textbf{P3}}, there is no need to explicitly store or enumerate the partitions as long as the correct binary vector is returned.

In practice, we rarely have access to the true model $M$ to be able to perform arbitrary interventions and evaluate the above properties, and instead have a dataset of observations and outcomes $\mathcal{D} = \{(\mathbf{s}_i, \mathbf{y}_i)\}_{i=1}^N$. In place of the structural causal model $M$, we learn a model of the environment $f(\mathbf{s}, \mathbf{b}; \phi)$ that takes in a state $\mathbf{s}$ and a binary vector $\mathbf{b}$, and returns a value for the outcome $\mathbf{Y}$ that only depends on the variables $\mathbf{S}_\mathbf{b}$ corresponding to the ones in $\mathbf{b}$. In practice, this can be implemented by using $\mathbf{b}$ to mask out all variables except for those in $\mathbf{S}_\mathbf{b}$.

We describe our approach for jointly learning a model of the environment and mapping states to binary vectors in Algorithm \ref{joint-optimization}. The model is trained to enforce minimality (\hyperref[property:minimality]{\textbf{P3}}) by minimizing the length of the binary vectors corresponding to every state observed in the dataset. Since $h_Y(\mathbf{s}; \theta)$ only returns one binary vector for every possible state, each binary vector corresponds to a unique partition of the state space. The invariance property (\hyperref[property:minimality]{\textbf{P1}}) is enforced by constraining the output of model $f(\mathbf{s}, \mathbf{b}; \phi)$ to match the value $f_\mathbf{Y}(\mathbf{s}, \mathbf{u})$ returned by the original causal model $M$, since this implies that only the variables $\mathbf{S}_\mathbf{b}$ are needed to predict the outcome $\mathbf{Y}$. Adding this as an optimization constraint and using its Lagrangian dual form results in the objective in Equation \ref{optimization-learned-model}. 
\vspace{-0.4cm}
\begin{equation}
\begin{split}
\label{optimization-learned-model}
    \min_{\theta, \phi} \sum_{i = 1}^{N} \bigg( |h_\mathbf{Y}(\mathbf{s}^{(i)}; \theta)| + \beta \lVert f(\mathbf{s}^{(i)}, \mathbf{b}^{(i)}; \phi) - \mathbf{y}^{(i)}  \rVert \bigg) \\[-.2cm] \text{ where } \mathbf{b}^{(i)} = h_\mathbf{Y}(\mathbf{s}^{(i)}; \theta)
\end{split}
\end{equation}
\vspace{-.5cm}
\begin{algorithm}
\caption{\algoname{} (\algoacronym{})}
\begin{algorithmic}
  \STATE {\bfseries Input:} Dataset $\mathcal{D} = \{(\mathbf{s}^{(i)}, \mathbf{y}^{(i)}\}_{i=1}^N$
  \STATE {\bfseries Output:} state-binary mapping $h_\mathbf{Y}(;\theta)$ 
  \REPEAT
    \STATE $\mathbf{b}^{(i)} = h_\mathbf{Y}(\mathbf{s}^{(i)}; \theta)$ \textbf{for} $i \in \{1,...,N\}$
    \STATE $\phi := \argmin_\phi \sum_{i = 1}^{N} \lVert f(\mathbf{s}^{(i)}, \mathbf{b}^{(i)}; \phi) - \mathbf{y}^{(i)}  \rVert$
    \STATE $\theta := \argmin_\theta \sum_{i = 1}^{N} \bigg( |h_\mathbf{Y}(\mathbf{s}^{(i)}; \theta)| + \beta \lVert f(\mathbf{s}^{(i)}, \mathbf{b}^{(i)}; \phi) - \mathbf{y}^{(i)}  \rVert \bigg)$
  \UNTIL{$(\theta, \phi)$ converges}
\end{algorithmic}
\label{joint-optimization}
\end{algorithm}
\vspace{-0.5cm}

\algoacronym{} thus learns to identify functional actual causes by jointly optimizing the masked forward model $f$ and state-binary mapping $h_\mathbf{Y}$. Appendix \ref{AppendixNeuralModel} contains additional details on the neural network architecture used for our implementation of \algoacronym{}, how necessity is enforced, and visualizations. 

\section{Experiments}
\label{sec:experiments}
In this section, we evaluate Algorithm~\ref{joint-optimization} on a set of complex, continuous-valued domains. We first assess how well \algoacronym{} recovers actual causes under different causal structures using a set of \textit{Random Vectors} domains. We also evaluate its performance in the reinforcement learning (RL) domains \textit{Mini-Breakout} and \textit{2D Robot Pushing}.


\subsection{Baselines}

Traditional methods in actual causality are not feasible to verify in the above domains, since they require a perfect SCM as well as enumerating all counterfactual states in a continuous and high-dimensional space. However, several other approaches in RL focused on discovering \textit{local causality} or \textit{interactions} are essentially targeting the same problem of identifying state factors that actually influenced the outcome at an observed state. These methods have previously been evaluated on downstream RL tasks, but not directly on their accuracy in identifying the presence of interactions. We compare \algoacronym{} with the following baselines:

\vspace{-0.3cm}
\begin{enumerate}
\itemsep0pt
    \item \textit{Gradient-based inference} \citep{wang2023elden} \textbf{Grad}: This approach uses input gradient magnitude to identify actual causes by learning a model $f(\mathbf{s};\phi)$ to predict the outcome variable from the inputs. For an observed state $\mathbf{s}$, each component of the output binary vector is given by $\mathbf{b}_i \coloneqq \mathbbm{1}\left(\left\|\frac{\delta \mathbf{y}}{\delta \mathbf{s}_i}\right\|_1 > \tau\right)$.
    
    \item \textit{Attention-based inference} \citep{rahaman2022neural} \textbf{Attn}: This approach learns a multi-head attention model $f(\mathbf{s};\phi)$ that maps input states to outcomes, then assigns components of the input as actual causes based on the total attention. The attention model is trained with an entropy regularization that encourages the heads to focus on a single input. For an observed state $\mathbf{s}$, each component of the output binary vector is given by $\mathbf{b}_i \coloneqq \mathbbm{1}\left(\frac{\sum_{k=1}^Na_i^k}{N} > \tau\right)$, where $a_i^k$ is the attention on $\mathbf{s}_i$ from head $k$ and $N$ is the number of heads.
    
    \item \textit{Counterfactual inference} \citep{pitis2020counterfactual} \textbf{CF}: For each causal variable $\mathbf{S}_i$ we sample $N$ random values and counterfactually assign $\mathbf{S}_i=\mathbf{s}_i'$ to obtain $N$ states $\{\mathbf{s}_i^{(1)},..., \mathbf{s}_i^{(N)}\}$. We then learn a model $f(\mathbf{s};\phi)$ that maps input states to outcomes, and test if the variance of the outcome under this model across the states $\{\mathbf{s}_i^{(1)},..., \mathbf{s}_i^{(N)}\}$ is above a threshold $\tau$. Each component of the output binary vector is given by $\mathbf{b}_i \coloneqq \mathbbm{1}\left(\frac{1}{N}\sum_{n=1}^N \| f(\mathbf{s}_i^{(n)}; \phi) \|_1 > \tau\right)$.
\end{enumerate}

\vspace{-0.3cm}
We report the accuracy of the optimal $\tau$ for each setting and each baseline. Further details on these baselines are in Appendix \ref{Baselines}. We also illustrate how the FAC definition unifies the heuristics employed by the above methods as assumptions about the invariance, necessity, and minimality properties in Appendix~\ref{Appendix:localcausality}.


\subsection{Random Vectors}

A Random Vectors domain is defined by a set of active variables $\mathcal{X} = \{X_1, \hdots X_n, Y\}$, where $Y$ is the outcome variable, as well as a set of passive variables $\tilde{\mathcal{X}} = \{\tilde{X}_1, \hdots \tilde{X}_n, \tilde{X}_Y\}$. All variables in the model are $d$-dimensional vectors and take on values in $[-1,1]^d$. The model determines their values with $R$ relations: linear functions of the values of parent variables $\mathcal{P}_r \subseteq \mathcal{X} \cup  \tilde{\mathcal{X}}$ whose output is one of the active variables in $\mathcal{X}$. If there are $R_t$ relations whose output is the active variable $X_t$ then each of these relations returns a value in $[-\frac{1}{R_t},\frac{1}{R_t}]^d$

Each active variable $X_t \in \mathcal{X}$, including the outcome, has a corresponding passive variable $\tilde X_t$. For each relation where the output variable is $X_t$, the returned value always depends on the passive variable, but it may or may not depend on the active variables based on their current values. We aim to infer the actual causes of outcome variable $Y$.

Each relation $f_r$ is associated with four matrices $\mathbf{A_r, B_r, C_r, D_r}$, and is a conditional linear relationship between the parent variables $\mathbf{P}_r$, and a target variable $X_t$. Let $x_a = [x_i, \hdots]_{i \in \mathcal P_r}$ denote the concatenation of all the parent variable values, and $[x_a, \tilde x_t]$ denote the concatenation of the parent variable values and passive variable value. Then the output of relation $f_r$ is given by: 

\begin{equation*}
\begin{split}
f(x_a, x_t) \coloneqq \mathbbm{1}(\mathbf{D}[x_a, \tilde x_t] \le \tau)(\frac{b}{\sqrt{d}}\mathbf{C}\tilde x_t) \quad +\quad \\ \mathbbm{1}(\mathbf{D}[x_a, \tilde x_t] > \tau)(\frac{b}{\sqrt{d} \cdot (|\mathcal P| + 1)}\left(\mathbf{A}x_a + \mathbf{B}\tilde x_t\right))
\end{split}
\end{equation*}

\begin{table}
\begin{small}
\begin{tabular}{@{}lcccc@{}}
\toprule
 Env & \algoacronym{} & Grad & Attn & CF \\
\midrule
\textbf{1-in} & $\mathbf{3.6\pm 0.1}$ & $8.8\pm 3.6$& $48 \pm 0.8$ & $42\pm 0.9$\\
\textbf{2-in} & $\mathbf{2.6\pm 2.0}$ & $27\pm4.0$& $50\pm 1.1$ & $42\pm 1.5$\\
\textbf{3-in} & $\mathbf{4.5\pm 1.9}$ & $41\pm 3.8 $ & $47\pm3.7$ & $44\pm 1.1$ \\ 
\textbf{3-chain} & $\mathbf{9.3\pm 7.6}$ & $46\pm 3.2$& $46\pm 0.5$ & $47\pm 0.5$\\
\textbf{3-m-in} & $\mathbf{2.0\pm 0.7}$ & $17\pm 3.5$ & $50\pm0.5$ & $44\pm 0.3$\\ 
\textbf{d-20} & $\mathbf{7.7\pm 1.2}$ & $\mathbf{8.0\pm 0.4} $ & $43\pm0.8$ & $45\pm 1.1$ \\ 
$\tau$\textbf{-1} & $\mathbf{0.2\pm 0.0}$ & $2.1\pm 0.3 $ & $9.9\pm0.8$ & $7.7\pm 0.5$ \\ 
\textbf{Push} & $\mathbf{13\pm 1.3}$ & $22\pm 1.1 $ & $39\pm0.6$ & $41\pm 1.2$ \\ 
\textbf{Break} & $\mathbf{3.6\pm0.2}$ & $\mathbf{10\pm 7.1} $ & $16\pm6.5$ & $35\pm 1.1$ \\ 
\bottomrule
\end{tabular}
\caption{Error (\%) in identifying ground truth actual causes in each of the domains.}
\label{AccuracyTable}
\end{small}
\vspace{-0.7cm}
\end{table}
Here $\tau$ is a threshold hyperparameter, such that the value of $X_t$ depends only on $\tilde x_t$ if $\mathbf{D}[x_a, \tilde x_t] \le \tau$, but depends on both $x_a$ and $\tilde x_t$ if $\mathbf{D}[x_a, \tilde x_t] > \tau$. The value of the indicator $\mathbf{D}[x_a, \tilde x_t] > \tau$ therefore determines whether the active variables are actual causes or not, and serves as ground truth for our evaluation. When $\tau$ is smaller, the active variables are actual causes less frequently. 

We evaluate \algoacronym{} on a number of different domains given by the following causal graphs: \\

\vspace{-0.7cm}
\begin{itemize}
    \item \textbf{1-in}: $X_{a, 1}\rightarrow Y$
    \item \textbf{2-in}: $X_{a,1}\rightarrow Y, X_{a,2}\rightarrow Y$ 
    \item \textbf{3-in}: $X_{a,1}\rightarrow Y, X_{a,2}\rightarrow Y, X_{a,3}\rightarrow Y$ 
    \item \textbf{3-chain}: $X_{a,1}\rightarrow X_{a,2} \rightarrow Y$
    \item \textbf{3-m-in}: $[X_{a,1}, X_{a,2},X_{a,3}]\rightarrow Y$
    \item \textbf{d-20}: $X_{a,1}\rightarrow Y$ with $d =20$ instead of $4$
    \item $\tau$-\textbf{1}: $X_{a,1}\rightarrow Y$ with $<10\%$ relation frequency
\end{itemize}

Aside from \textbf{d-20}, all variables are 4-dimensional. The number in \textbf{1-in}, \textbf{2-in}, and \textbf{3-in} represents the number of incoming edges. \textbf{3-chain} is a chain graph with length 3. For \textbf{3-m-in}, $(X_{a,1}, X_{a,2}, X_{a,3})$ are the parents in a single relation. \textbf{d-20} is a variant of \textbf{1-in} where the dimensionality of each variable is increased to 20. In $\tau$-\textbf{1}, the threshold hyperparameter is set to -1 such that the active variables are actual causes much less frequently ($<10\%$). Additional details on these domains can be found in Appendix \ref{appendix:random-vectors}. 

We generate $100k$ states from each domain, using $90k$ of them as the training set and $10k$ as the test set. Appendix~\ref{JACITrainingDetails} contains additional training curves and a false-positive/negative assessment. As shown in Table \ref{AccuracyTable}, in all domains our approach demonstrates significantly higher accuracy at identifying actual causes. 


\subsection{Mini-Breakout and 2D Pushing}

To evaluate how well \algoacronym{} scales to more realistic domains, we tested its performance on \textit{Mini-Breakout} and \textit{2D Pushing}. 2D Pushing is a domain similar to Example~\ref{block-pushing-example} where an agent pushes a block in a continuous 2D space, and both the agent and the block can be impeded by obstacles. Mini-Breakout is a domain where a ball must be bounced off of a paddle to hit bricks. This domain is challenging because the ball makes infrequent, complex interactions with many objects while following its own passive dynamics. These domains are illustrated in Figures \ref{fig:push2d} and \ref{fig:break} respectively. 

The dynamic causal variables in 2D Pushing are [Action, Gripper, Block, Target, Obstacle1, $\hdots$], where the block position is used as the outcome variable. In Breakout the causal variables are [Action, Paddle, Ball, Brick1, Brick2, $\hdots$], where the ball position is used as the outcome variable. Because the simulator indicates when objects interact in these domains, it can be used to identify ground truth interactions. We discuss additional details and the causal graphs of these domains in Appendix ~\ref{2dpusher} and Appendix ~\ref{breakout}.

As we observe in Table~\ref{AccuracyTable}, \algoacronym{} is significantly more accurate than all baselines at identifying actual causes. The main limitation of \algoacronym{} is that without hyperparameter tuning, it can converge to local minima such as the binary vector of all zeros (no actual causes) or all ones (the full state is an actual cause). Additional tuning discussion is in Appendix~\ref{JACITrainingDetails}.

\begin{figure}
\centering     
\subfigure[]{\label{fig:push2d}\includegraphics[width=0.45\linewidth]{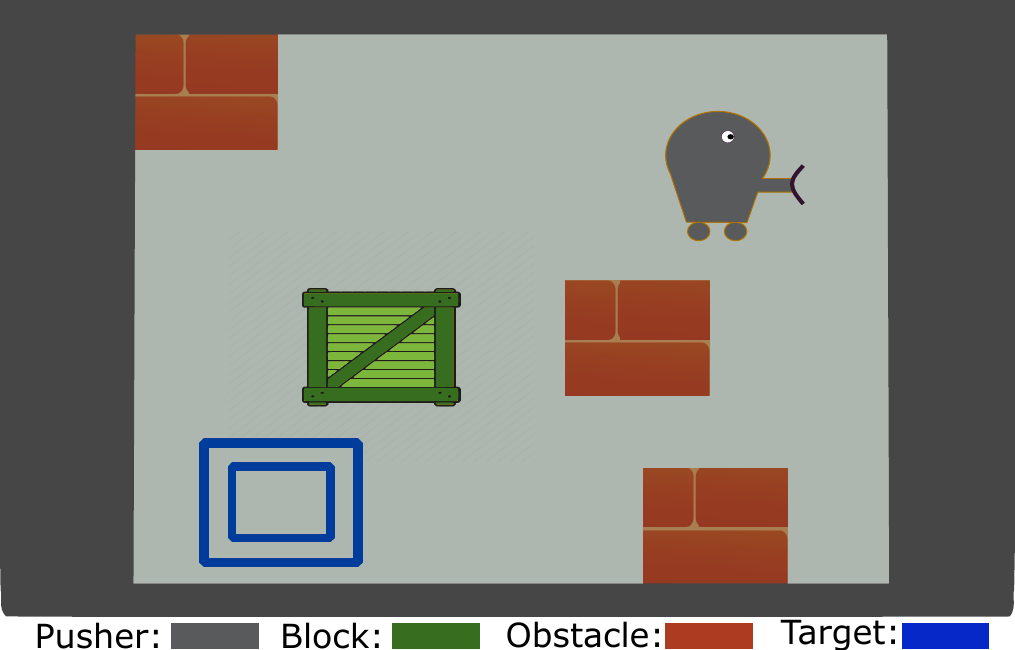}}
\subfigure[]{\label{fig:break}\includegraphics[width=0.45\linewidth]{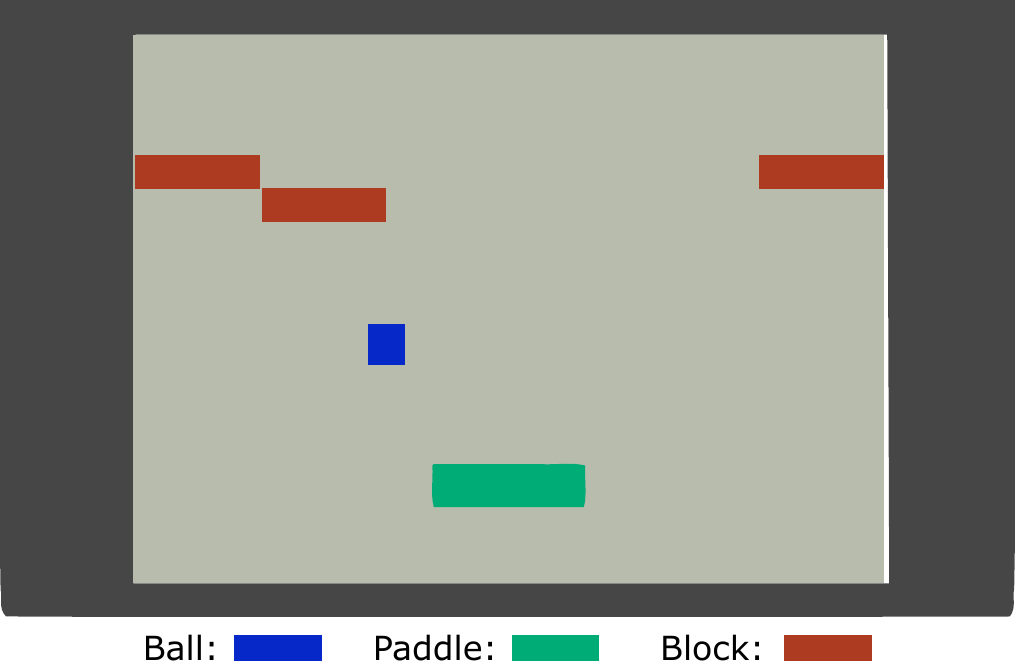}}
\vspace{-0.2cm}
\caption{\textbf{(a)} \textbf{2D Pushing}: the continuous actions move the pusher in $[-1,-1]^2$ in a $5$ by $5$ area. \textbf{(b)} \textbf{Mini-Breakout}: the actions move the paddle to the right and left.}
\vspace{-0.9cm}
\label{fig:1dmover}
\end{figure}

\vspace{-0.3cm}
\section{Conclusion}
\label{conclusion}

In this work, we make a step towards bridging the theoretical notion of actual causation with the practical objective of automated inference and scaling to complex environments. The \textit{functional actual cause (FAC)} definition takes into account the context-specific independencies that exist in many real-world environments and is accordingly more restrictive than existing definitions. To demonstrate the practicality of this approach, we develop a continuous optimization-based strategy for inferring actual causes, \textit{\algoname{} (\algoacronym{})}, and show that it performs effectively while scaling to continuous-valued and complex environments.

\section{Acknowledgements}

We thank Eliza Stefaniw and Purva Pruthi for their input on various drafts of this paper. This work has taken place in part in the Safe, Correct, and Aligned Learning and Robotics Lab (SCALAR) and in part at the Knowledge Discovery Laboratory (KDL) at The University of Massachusetts Amherst. SCALAR research is supported in part by the NSF (IIS-2323384), AFOSR (FA9550-20-1-0077), ARO (78372-CS), and the Center for AI Safety (CAIS). KDL research is supported in part by the Defense Advanced Research Projects Agency (DARPA) and the Army Research Office (ARO) and was accomplished under Cooperative Agreement Number W911NF-20-2-0005. This research was supported by the Office of Naval Research (ONR) through the National Defense Science and Engineering Fellowship (NDSEG). The views and conclusions contained in this document are those of the authors and should not be interpreted as representing the official policies, either expressed or implied, of the DARPA or ARO, or the U.S. Government. The U.S. Government is authorized to reproduce and distribute reprints for Government purposes notwithstanding any copyright notation herein.


\bibliography{references}
\bibliographystyle{icml2024}
\newpage
\appendix
\onecolumn


\section{Related Work}
\label{related-work}

Our work unifies definitions of actual causality with practical inference algorithms for machine learning and sequential decision-making and touches on a wide variety of related work, some of which we highlight in this section.


\subsection{Analogous Concepts in Causal Inference}
The problem of the obstacle being an actual cause in Example \ref{block-pushing-example} relates to the classic Queen of England problem: \textit{the queen of England's not watering my flowers caused my flowers to wilt} is contrary to most human judgments, but could technically be correct depending on the definition of actual cause. \citep{schaffer2005contrastive} connects this problem to the ambiguity between causation and non-causation, i.e. the non-occurrence of events being causes. This idea is related to \textit{normality}~\citep{halpern2015graded}, which has been observed in human behavior~\citep{icard2017normality,kominsky2015causal,morris2018judgments} and formally expanded on~\citep{blanchard2017normality}. Normality has also been incorporated in causal probabilistic logic~\citep{beckers2015combining}, but rarely in continuous domains or for inference with learned models. 

Actual causality has well-established relationships to abnormality~\citep{hilton1986knowledge}, which is built on the basic causal assumption of describing an actual cause as present \textit{when present and absent when absent}~\citep{kelley1980attribution}. Nonetheless, the subject of normality has remained challenging to formalize, because human intuitions appear to be some combination of prescriptive (something is the cause because it should have been some way) versus statistical (something is the cause because it is unlikely)~\citep{hitchcock2007prevention,menzies2007causation,hitchcock2009cause}. In some cases, the prescriptive norms are described according to moral or legal reasoning~\citep{kominsky2019immoral,lagnado2017causation,knobe2021proximate}, and in other cases, in terms of mechanics~\cite{bear2017normality,samland2016prescriptive}. In cases such as these, there are often few solutions more elegant than simply constructing a heuristic gradation of values ~\citep{halpern2015graded}. FAC instead represents normality using the invariant preimage. Perhaps the closest line of research to FAC investigates the statistical properties of normality, including the relationship with statistical correlation in counterfactual space~\cite{quillien2020we,quillien2023counterfactuals,kirfel2021causal}. However, FAC extends these definitions by applying context-specific independence and set-partitioning minimality, which formalizes notions of prescriptive causes with statistical correlation. 

Outside of the context of normality, the concept of the invariant preimage and identifying subsets of the state space where specific invariances exist is analogous to representing and learning context-specific independencies \citep{neville2007relational}, which has been explored in prior work on probabilistic graphical models \citep{shen2020new, boutilier2013context}. Mixtures of DAG representations \citep{strobl2022causal}, especially using dynamic networks \citep{dagum1995uncertain,zhang2017causal} frame the invariance using the generative DAG model. While actual causality introduces additional constraints related to necessity, Algorithm~\ref{joint-optimization} bears some resemblance to methods modeling mixture distributions~\citep{bellot2021consistency, lu2021improving}. 

Learning to identify independence to certain variables in general has been useful for a variety of applications, including RL. This is the case in machine learning tasks in theory~\citep{ferns2012metrics, nachum2021provable} and in practice for tasks such as dynamics modeling~\citep{li2018learning,feng2023learning}, decision-making~\citep{zhang2020invariant}, representation learning~\citep{oord2018representation}, and inference in graphical models~\citep{poole2003exploiting}. However, these methods can only operate in specific contexts where many causal variables are generally irrelevant. When there are only a few specific contexts where variables are independent, these methods provide a much more limited benefit.

\subsection{Context Specific Independence}

Causal relationships that are context-specific have been previously studied in the literature on causal graphical models, where they are formalized using various concepts such as context-specific independence \citep{poole2003exploiting, boutilier2013context}, partial conditional independence \citep{pensar2016role}, and context-set specific independence \citep{hwang2023discovery}. In these settings, sets of variable assignments in which particular independence relationships hold are referred to as \textit{domains} or \textit{context sets}. The concept of invariant preimage is analogous to these notions, as it represents regions of the state space where the outcome is invariant to certain inputs. A key distinction between invariant preimages and context-specific independence (CSI) is that the former asserts assumes invariance only within a subset of states while the latter assumes complete independence when certain variables are fixed to specific values, though some work~\citep{boutilier2013context,hwang2023discovery} has explored relaxing this constraint. 

CSI addresses a problem similar to that of actual causality; it describes the contextually \textit{independent} parents in a particular state, while actual causality describes what can intuitively be interpreted as the contextually \textit{dependent} parents of an outcome. Actual causation emphasizes what \textit{actually happened} and not just any context, which is a requirement for the counterfactual in actual causation to make sense. Additionally, CSI aims to encompass the spectrum between partial conditional independence and marginal independence. The former describes an independence relationship that holds for a single assignment of the other variables, while the latter describes an independence relationship that holds regardless of the values of other variables. Similarly, Theorem \ref{thm-fac-direct-sufficiency} describes how functional sufficiency captures the spectrum between direct sufficiency and weak sufficiency. By synthesizing the concepts of necessity, sufficiency, and witness sets from actual causation with the notions of CSI and context sets as state-space partitions, this work provides a novel contribution to both fields.

The complexity of identifying context-specific independence is NP-hard~\citep{corander2019logical}, as is that of identifying blame and responsibility~\citep{aleksandrowicz2017computational} with actual causation. In practice, context-specific independence has been inferred using staged trees~\citep{leonelli2023context}, Bayesian networks with Boolean functions~\citep{zou2017representing}, weighted adjacency matrices~\citep{brouillard2020differentiable} and stratified Gaussian models~\cite{nyman2017stratified}. Additionally, the subset partitioning for defining context sets in CSI is exactly formulated in \textit{local partial context-specific independence} (LP-CSI)~\citep{pensar2016role}, and subsequent work has used Bayesian methods~\citep{pensar2018exact} and logical graph structures~\citep{tikka2019identifying} to recover these partitions. LP-CSI captures the notion of the invariant preimage and tries to recover the set of invariant parents for different states. Though an analogy can be drawn between Bayesian methods and minimality (AC3) since they regularize the model, they do not capture the notion of necessity or the witness set. 

\algoacronym{} (Algorithm \ref{joint-optimization}) can be seen as a function approximation method for learning these partitions that incorporates necessity (AC2a) in the structure of the function approximator. This is especially relevant in similarities to Neural Contextual Decomposition (NCD) \citep{hwang2023discovery}, which utilizes a forward model masked by outputs from a Gumbel-Bernoulli binary variable. However, NCD does not use joint optimization, regularization, or order-invariant structure to infer LP-CSI graphs. On the other hand, \citet{hwang2023quantized} utilize a quantized codebook with regularization and a joint optimization to infer LP-CSIs, though without order-invariant structure or adaptive optimization of \algoacronym{}. The order-invariant structure required to capture the necessity condition in actual causation is what sets JACI apart from these approaches, as it is not a requirement for CSI.


\subsection{Causal Explanations and Decision Making}
Generating causal explanations is one of the most intuitive applications of actual causality \citep{beckers2022xai} and a recurring focus of prior work, extending to problem settings such as classification \citep{bertossi2020causality}, answering queries \citep{meliou2010complexity} and language models \citep{kıcıman2023llm}. For RL domains, past work has focused on explaining the decisions of RL policies \citep{nashed2023causal}, evaluating the degree of responsibility of particular decisions \citep{triantafyllou2022actual} and team games \citep{alechina2020causality}, and studying human participants' explanations of the behavior of RL agents \citep{madumal2020explainable}. In general, this work draws from but does not directly relate to explanation, though there are several opportunities in this area for future work.


\subsection{Sequential Decision Making and Reinforcement Learning}
Actual causality has many promising directions for applications in RL \citep{buesing2018woulda}, though they are described using varied terminology including \textit{interaction} \citep{yang2023learning}, \textit{local causality}~\citep{pitis2020counterfactual}, \textit{controllability}~\citep{seitzer2021causal} and \textit{counterfactuals}. Using actual causes as a signal for exploration has been explored using counterfactual data augmentation in \citet{pitis2020counterfactual}, which introduces the notion of \textit{local causality}. This uses measures such as proximity to determine actual causes, then counterfactually alters them to generate new data. This work has been extended to identify regions where a planner~\citep{chitnis2021camps} or policy can perform well~\citep{pitis2022mocoda}. Alternative methods utilize interventional data \citep{baradel2019cophy} applied at object interactions. 

Causal graphical models have been applied to reinforcement learning with some success~\citep{zhang2020invariant,zhang2020learning,wang2022causal,ding2022generalizing}. Further extensions to these works have touched upon actual causality by various means, such as identifying controllable causal states \citep{seitzer2021causal}, identifying local causes with model gradients \citep{wang2023elden}, graph networks~\citep{feng2023learning} and using curiosity-based causal graph methods \citep{sancaktar2022curious}. Interestingly, though some of these works incorporate dynamic graphs, they often lean heavily into false positives, a result we validate in our baselines. 

Actual causes described as interactions are also useful for skill learning, as shown by prior work on identifying actual causes with pairwise models \citep{chuck2020hypothesis,chuck2023granger} and action counterfactuals \citep{choi2023unsupervised}. A model of the environment that encodes partial invariances can be useful for planning, and recent work attempts to do this using causal structure learning \citep{ding2022generalizing,wang2022causal,rahaman2022neural}. By jointly learning a model of the environment and inferring actual causes, this work offers applications to all of these directions. We include an extended discussion of future directions and specific applications of actual causality to RL in Appendix~\ref{connections}.

\section{Proof of Theorem 4.6}
\label{appendix-proofs}

\setcounter{theorem}{5}
\begin{theorem}
    For an observed state $\mathbf{s^*}$ and binary-subset pair $(\mathcal{B}(\mathbf{Y}), \mathcal{I}(\mathbf{Y}))$, if the event $\mathbf{X} = \mathbf{x}$ satisfies AC1, functional sufficiency, and contrastive necessity for the outcome $\mathbf{Y} = \mathbf{y}$ under witness set $\mathbf{W}$, and $(\mathcal{B}(\mathbf{Y}), \mathcal{I}(\mathbf{Y}))$ is minimal, then there cannot exist a smaller set of variables $\mathbf{Z} \subset \mathbf{X}$ such that $\mathbf{Z} = \mathbf{z}$ also satisfies AC1, functional sufficiency, and contrastive necessity for $\mathbf{Y} = \mathbf{y}$, where $\mathbf{z}$ denotes the values of variables $\mathbf{Z}$ in the assignment $\mathbf{x}$.
\end{theorem}

\begin{proof}
    Say that $\mathbf{s^*}$ lies within the invariant preimage $\mathcal{I}_k$. By definition the invariant preimage satisfies Equation \ref{IVP definition}, and therefore for every state $\mathbf{s} \in \mathcal{I}_k$ we will have:
    
    \begin{equation*}
        f_\mathbf{Y}(\mathbf{s_X}, \mathbf{s_C}, \mathbf{u}) = f_\mathbf{Y}(\mathbf{s_X}, \mathbf{s^\prime_C}, \mathbf{u}) \quad \forall \: \mathbf{s^\prime} \in \mathcal{I}_k \text{ where } \mathbf{s^\prime_X} = \mathbf{s_X}.   
    \end{equation*}

    By definition, its corresponding binary vector $b_{k}$ will have one element equal to one for every variable in $\mathbf{X}$. We first show that $\mathbf{s^*}$ could not have been mapped to an invariant preimage with another binary vector $b_j$ such that $|b_j| > |b_k|$.

    Say that $\mathbf{s^*}$ could have instead been mapped to some binary vector $b_j$ with a longer length, i.e. $|b_j| > |b_k|$. By definition there is only one invariant preimage corresponding to a given subset of exogenous variables, and only one binary vector corresponding to each invariant preimage, so $\mathbf{s^*}$ would have to be moved to the invariant preimage $\mathcal{I}_j$ corresponding to the binary vector $b_j$.

    Let $L_1$ denote the cost of the given binary-subset pair, i.e. $L_1 = \sum_{k=1}^K |\mathcal{I}_k||b_k| = |\mathcal{I}_j||b_j| + |\mathcal{I}_k||b_k| + (|\mathcal{I}_1||b_1| + \cdots |\mathcal{I}_K||b_K|)$
    
    If the state $\mathbf{s}$ was moved from $\mathcal{I}_k$ to $\mathcal{I}_j$ the resultant cost would be $L_2 = (|\mathcal{I}_j| + 1)|b_j| + (|\mathcal{I}_k| - 1)|b_k| + (|\mathcal{I}_1||b_1| + \cdots |\mathcal{I}_K||b_K|) = |b_j| - |b_k| + L_1$.
    
    We know that $|b_j| > |b_k|$ and hence $L_2 > L_1$. Therefore, assigning $\mathbf{s^*}$ to the binary vector $|b_k|$ will have lower cost, and by extending this to all states $\mathbf{s} \in \mathcal{I}_k$ we will have $|b_k| \le |\mathbf{X}|$.
    
    Now say that $\mathbf{s^*}$ is instead mapped to some binary vector $b_j$ with a shorter length, i.e. $|b_j| < |b_k|$, and its corresponding invariant preimage is $\mathcal{I}_j$. If moving this state to $\mathcal{I}_j$ leads to Equation \ref{IVP definition} not being satisfied or necessity (Definition \ref{defn-necessity}) not being satisfied, the resultant set of invariant preimages is not valid and therefore $\mathbf{Z} = \mathbf{z}$ cannot be a functional actual cause. However, if the aforementioned properties are satisfied then since $\mathbf{s^*}$ is not in $\mbox{\textit{IVP}}(\mathbf{X}; \mathbf{Y})$, either $\mathbf{X} = \mathbf{x}$ could not have been functionally sufficient for $\mathbf{Y} = \mathbf{y}$ or the given binary-subset pair was not minimal, both of which are contradictions. Hence we must have $|b_k| \ge |\mathbf{X}|$.

    Since $|b_k| \le |\mathbf{X}|$ and $|b_k| \ge |\mathbf{X}|$, we have that $|b_k| = |\mathbf{X}|$. Therefore, any actual cause that satisfies functional sufficiency and contrastive necessity with a minimal binary-subset pair must have cardinality $|\mathbf{X}|$. We note that this is true regardless of which witness set is used to determine if $\mathbf{Z} = \mathbf{z}$ was necessary since functional sufficiency would be violated either way. As a result, $\mathbf{Z} = \mathbf{z}$ cannot be a functional actual cause. 
    
\end{proof}
\section{Direct Sufficiency and Proof of Theorem \ref{thm-fac-direct-sufficiency}}
\label{appendix: direct-sufficiency}

\cite{beckers2021actual} provides a definition of direct sufficiency, which imposes a stronger constraint by requiring that the outcome is always $\mathbf{Y} = \mathbf{y}$ when the intervention $\mathbf{X} \leftarrow \mathbf{x}$ is applied, regardless of the values of all other variables $\mathbf{C}$. In other words, the outcome should be robust to all possible interventions on the remaining variables $\mathbf{C} = \mathbf{S} - \mathbf{X}$. 

\begin{definition}[Direct Sufficiency]
\label{defn-direct-sufficiency}
    For an observed state $\mathbf{s}^*$, the event $\mathbf{X} = \mathbf{x}$ is directly sufficient for the outcome $\mathbf{Y} = \mathbf{y}$ if for all $\mathbf{c} \in \mathcal{R}(\mathbf{C})$ we have that $(M, \mathbf{u}) \vDash [\mathbf{X} \leftarrow \mathbf{x}, \mathbf{C} \leftarrow \mathbf{c}] \mathbf{Y} = \mathbf{y}$.
\end{definition}

Invariant preimages are defined such that the invariance to all remaining variables $\mathbf{C}$ only needs to hold for a subset of possible assignments of $\mathbf{C}$, given by the states in $\mbox{\textit{IVP}}(\mathbf{X}; \mathbf{Y})$, while direct sufficiency requires that the invariance holds for all assignments of $\mathbf{C}$. In other words, while functional sufficiency is a statement about context-specific invariance, direct sufficiency is a statement about complete invariance with respect to $\mathbf C$. To illustrate this, we define a more strict notion of invariance as follows:

\begin{definition}[Complete Invariance]
    An invariant preimage $\mbox{\textit{IVP}}(\mathbf{X}; \mathbf{Y})$ shows \textit{complete invariance} with respect to a state $\mathbf{s}^*$ if all states $\mathbf{s}^\prime$ where $\mathbf{s^\prime_X} = \mathbf{s^*_X}$ also lie within $\mbox{\textit{IVP}}(\mathbf{X}; \mathbf{Y})$.
\end{definition}

It can be verified that if complete invariance holds for one state $\mathbf{s} \in \mbox{\textit{IVP}}(\mathbf{X}; \mathbf{Y})$ then it should hold for all other states in $\mbox{\textit{IVP}}(\mathbf{X}; \mathbf{Y})$. We now complete the proof of Theorem \ref{thm-fac-direct-sufficiency}.

\begin{theorem}
    For an observed state $\mathbf{s}^*$, functional sufficiency and direct sufficiency are equivalent when the invariant preimage $\mbox{\textit{IVP}}(\mathbf{X}; \mathbf{Y})$ that contains $\mathbf{s}^*$ shows complete invariance, i.e. all other states $\mathbf{s}^\prime$ where $\mathbf{s^\prime_X} = \mathbf{s^*_X}$ also lie within $\mbox{\textit{IVP}}(\mathbf{X}; \mathbf{Y})$.
\end{theorem}

\begin{proof}

We prove Theorem~\ref{thm-fac-direct-sufficiency} by first demonstrating that a functionally sufficient actual cause under complete invariance implies a directly sufficient one, then that a directly sufficient actual cause implies a functionally sufficient one with complete invariance. 

If the event $\mathbf{X} = \mathbf{x}$ were functionally sufficient for the outcome $\mathbf{Y} = \mathbf{y}$ at the observed state $\mathbf{s}^*$, and the states corresponding to the assignment $\mathbf{X} = \mathbf{s^*_X}$ and all possible assignments of $\mathbf{C}$ were present within $\mbox{\textit{IVP}}(\mathbf{X}; \mathbf{Y})$, then any arbitrary intervention on $\mathbf{C}$ along with the intervention $\mathbf{X} \leftarrow \mathbf{x}$ would result in a state in $\mbox{\textit{IVP}}(\mathbf{X}; \mathbf{Y})$. By the definition of invariant preimage, the outcome would remain the same for all of these states. Hence $\mathbf{X} = \mathbf{x}$ would also be directly sufficient for the outcome $\mathbf{Y} = \mathbf{y}$.

Similarly, if $\mathbf{X} = \mathbf{x}$ were directly sufficient for the outcome $\mathbf{Y} = \mathbf{y}$ the outcome would need to remain the same under all possible assignments of the remaining variables $\mathbf{C}$. If the invariant preimage $\mbox{\textit{IVP}}(\mathbf{X}; \mathbf{Y})$ showed complete invariance with respect to the state $\mathbf{s}^*$, then all of the states corresponding to the assignment $\mathbf{X} = \mathbf{s^*_X}$ and all possible assignments of $\mathbf{C}$ would be present within $\mbox{\textit{IVP}}(\mathbf{X}; \mathbf{Y})$. One of these states is the observed state $\mathbf{s}^*$. Since direct sufficiency implies weak sufficiency (Proposition 2, \citet{beckers2021actual}) and $\mathbf{s}^*$ is in $\mbox{\textit{IVP}}(\mathbf{X}; \mathbf{Y})$, the event $\mathbf{X} = \mathbf{x}$ is functionally sufficient for the outcome $\mathbf{Y} = \mathbf{y}$.

\end{proof}

By explicitly encoding context-specific independence, functional sufficiency captures the spectrum between weak sufficiency (no invariance) and direct sufficiency (complete invariance). The above statements are expressed without the witness set for simplicity, but the same holds with a witness since context-specific independence and the assignment of states to invariant preimages are not affected by the choice of witness.

\section{Additional Examples}
\label{addn_ex}
In this section, we evaluate the Modified HP~\citep{halpern2016actual} and direct sufficiency~\citep{beckers2021actual} definitions on a few discrete examples that illustrate context-specific independence and illustrate the limitations of these definitions.
\begin{wraptable}{r}{6cm}
\vspace{-0.2cm}
\begin{small}
\begin{center}
\begin{tabular}{|l|c|r|}
\toprule
Pusher & Obstacle & Final\\
\midrule
499& not 501& 501\\
500& not 501 or 502& 502\\
499& 501& 500\\
500& 501& 500\\
500& 502& 501\\
not 499 or 500& all & 500\\
\bottomrule
\end{tabular}
\caption{\textbf{Discrete 1D Pusher} environment definition, mapping states to outcomes}
\label{Discpusherenv}
\vspace{-0.5cm}
\end{center}
\end{small}
\end{wraptable}
\citet{beckers2021actual} also introduces the strong sufficiency definition which operates on extended causal chains, and in future work we plan to extend FAC to this setting. 
However, in the following examples, strong sufficiency acts the same way as direct sufficiency since they all involve causes that are direct parents of the outcome.

\subsection{Discrete 1D Pusher}

This example is analogous to the 1D pusher domain in Example~\ref{block-pushing-example} but uses discrete states. We do not use a continuous state space domain because Modified HP and direct sufficiency do not scale to continuous state spaces in their present form.

The discrete 1D pusher domain contains three causal variables: pusher, obstacle, and final block position. The pusher and obstacle positions can vary in the set $\{0,\hdots,1000\}$, and the block position in range $\{500, 501, 502\}$. If the agent is within $\{499, 500\}$, the final block position will be $500 + (\text{agent position} - 500 + 2)$, unless the obstacle obstructs the block. Otherwise, the final block position remains constant at $500$. We illustrate the dynamics of this environment in Table~\ref{Discpusherenv}.

Suppose we are in the state corresponding to the first row of the table, where the pusher position is $499$, the obstacle position is $200$ and the final block position is $501$, and we want to determine whether the obstacle \textit{alone} is an actual cause. The pusher, obstacle, and final block positions are denoted as $P$, $X$, and $Y$ respectively, and there are no exogenous variables $U$ in this example. We can verify that the obstacle would be an actual cause according to the direct sufficiency and Modified HP definitions. \\

\textbf{Modified HP}:
\begin{enumerate}
    \item[] \textbf{AC1}: Satisfied since this is the observed state.
    \item[] \textbf{AC2a}: Satisfied for $M$ as the environment since $(M, \mathbf{u}) \vDash [X \leftarrow 501] Y \neq 501$ is true, as can be seen in the third row of the table. Note in this case the witness set is empty. 
    \item[] \textbf{AC2b}: Satisfied for $M$ as the environment since $(M, \mathbf{u}) \vDash [X \leftarrow 200] Y = 501$ is true, as can be seen in the first row of the table. 
    \item[] \textbf{AC3}: There is no strict subset of $X$ that could be an actual cause since the empty set does not satisfy AC2a.
\end{enumerate}

\textbf{Direct Sufficiency and Necessity}:
\begin{enumerate}
    \item[] \textbf{AC1}: Satisfied since this is the observed state.
    \item[] \textbf{AC2a}: Satisfied for $M$ as the environment since $(M, \mathbf{u}) \vDash [X \leftarrow 501,P\leftarrow 499] Y \neq 501$ is true, as can be seen in the third row of the table. This satisfies AC2a because this counterexample demonstrates that there is a counterfactual assignment that is not sufficient. In this case the pusher is in the witness set and set to its actual value (and is the same witness for sufficiency).
    \item[] \textbf{AC2b}: Satisfied for $M$ as the environment since $(M, \mathbf{u}) \vDash [X \leftarrow 200,P\leftarrow 499] Y = 501$ is true, as can be seen in the first row of the table. Note that the $C$ term for direct sufficiency $[X \leftarrow x,W\leftarrow w^{*}, C\leftarrow c]$ is satisfied because the set of variables $C$ in this case is empty. 
    \item[] \textbf{AC3}: There is no strict subset of $X$ that could be an actual cause since the empty set does not satisfy AC2a.
\end{enumerate}

Thus, both definitions imply that the obstacle alone is an actual cause \textit{in every state where the pusher pushes the block}, even though in $99.9\%$ of these cases, the outcome is contextually independent of the obstacle. While this conclusion is not technically wrong, a machine learning agent or downstream task utilizing the returned actual causes cannot take advantage of the context-specific independence. 

\begin{wraptable}{r}{0.4\textwidth}
\vspace{-.2cm}
\begin{small}
\begin{center}
\begin{tabular}{@{}lccr@{}}
\toprule
T1 & T2 & Bird & Launch\\
\midrule
$1$& $1$& not 500& $1$\\
$1$& $0$& not 500& $1$\\
$0$& $1$& not 500& $1$\\
$0$& $0$& not 500& $0$\\

$1$& $1$& 500 & $0$\\
$1$& $0$& 500 & $0$\\
$0$& $1$& 500 & $0$\\
$0$& $0$& 500 & $0$\\
\bottomrule
\end{tabular}
\caption{\textbf{Rocket Ship} environment}
\label{rocketshipenv}
\end{center}
\end{small}
\end{wraptable}
\subsection{Rocket Ship}
A rocket ship has two thrusters, represented by causal variables $T1$ and $T2$, which can turn on or off. If either one of the thrusters fires, the rocket ship will launch ($Y=1$). If neither thruster fires, the rocket ship will not launch ($Y=0$). There is also a bird ($X$) that is flying around the field, its position taking on values in $\{0,\hdots,1000\}$. There is one particular position, $X= 500$, where the bird knocks out the wiring and therefore causes the rocket to fail to launch, i.e. $Y=0$. 

Suppose we are in the state corresponding to the first row of the table, where $T1=1, T2=1$, the bird position is $X=200$, and $Y=1$. We want to determine if the bird position alone is an actual cause. Note that there are no exogenous variables $U$ in this example. We can verify that the bird position alone would be an actual cause according to the direct sufficiency and Modified HP definitions. \\

\textbf{Modified HP}:
\begin{enumerate}
    \item[] \textbf{AC1}: Satisfied since this is the observed state.
    \item[] \textbf{AC2a}: Satisfied for $M$ as the environment since $(M, \mathbf{u}) \vDash [X \leftarrow 500] \mathbf{Y} \neq 1$ is true, as can be seen in the fifth row of the table. Note in this case the witness set is empty.
    \item[] \textbf{AC2b}: Satisfied for $M$ as the environment since $(M, \mathbf{u}) \vDash [X \leftarrow 200] \mathbf Y = 1$ is true, as can be seen in the first row of the table. 
    \item[] \textbf{AC3}: There is no strict subset of the bird position alone whose assignment is an actual cause, since the empty set does not satisfy AC2a.
\end{enumerate}

\textbf{Direct Sufficiency and Necessity}:
\begin{enumerate}
    \item[] \textbf{AC1}: Satisfied since this is the observed state.
    \item[] \textbf{AC2a}: Satisfied for $M$ as the environment since $(M, \mathbf{u}) \vDash [X \leftarrow 500,T1\leftarrow 1] Y \neq 1$ is true, as can be seen in the fifth row of the table. This counterexample demonstrates that there is a counterfactual assignment of $X$ that is not sufficient. In this case, the first thruster is in the witness set.
    \item[] \textbf{AC2b}: Satisfied for $M$ as the environment since $(M, \mathbf{u}) \vDash [X \leftarrow 200, T1\leftarrow 1, T2\leftarrow 1] Y = 501$ is true and $(M, \mathbf{u}) \vDash [X \leftarrow 200, T1\leftarrow 1, T2\leftarrow 0] Y = 501$ is true, as can be seen in the first and second rows of the table. Note that the $C$ term for direct sufficiency is $T2$ since this is the only unassigned variable, and we vary over all its possible values ($1$ and $0$), and the outcome is still sufficient. As with AC2a, the first thruster is in the witness set.
    \item[] \textbf{AC3}: There is no strict subset of the bird position alone whose assignment is an actual cause, since the empty set does not satisfy AC2a.
\end{enumerate}

With both definitions, the bird position alone is an actual cause of the rocket launch when either thruster fires. If an agent encoded this form of invariance, it would assume that the bird was the reason the rocket launched, despite the outcome being dependent on the bird's position in only $0.1\%$ of states. In contrast, the \textit{functional actual cause} definition (Definition \ref{fac-defn}) would assign the values of (T1,T2) as a minimal actual cause in the observed state.

\subsection{Offensive Chatbot}

\begin{wraptable}{r}{0.3\textwidth}
\vspace{-0.1cm}
\begin{small}
\begin{center}
\begin{tabular}{@{}lccr@{}}
\toprule
User & Engineer & Bot\\
\midrule
$\textbf u$& not 500& $\textbf u$\\
$\textbf u$& 500& $0$\\
\bottomrule
\end{tabular}
\end{center}
\end{small}
\end{wraptable}

Spatial proximity can overlap with context-specific independence in many cases; for example, an object far away from an agent will often not influence its motion. However, this overlap is by no means complete, and spatial proximity or distances between counterfactual scenarios do not serve as a general proxy for context-specific independence. In this section, we provide an example where spatial proximity has no bearing on the outcome. 

A chatbot $C$ is designed to repeat (parrot) what the user $U$ inputs, which we represent by having $U$ take on values between $1-1000$ and assigning $C=U$. Suppose there is a software engineer $E$ who can do a variety of actions between $1-1000$, most of which have no effect. However, if they take the action $E=500$, then this will cause the chatbot to crash, and output $C=0$ regardless of what the user inputs. 

Suppose we are in the state: User at $500$, Engineer at $200$, so the chatbot is $500$, corresponding to the first row of the table. We want to determine if the Engineer alone is an actual cause. \\

\textbf{Modified HP}:
\begin{enumerate}
    \item[] \textbf{AC1}: Satisfied since this is the observed state.
    \item[] \textbf{AC2a}: Satisfied for $M$ as the environment since $(M, \mathbf{u}) \vDash [E \leftarrow 500] \mathbf{C} \neq 500$ is true, as can be seen in the second row of the table. Note in this case the witness set is empty.
    \item[] \textbf{AC2b}: Satisfied for $M$ as the environment since $(M, \mathbf{u}) \vDash [E \leftarrow 200] \mathbf C = 500$ is true, as can be seen in the first row of the table. 
    \item[] \textbf{AC3}: There is no strict subset of the engineer alone which is an actual cause since that would imply the empty set, which does not satisfy AC2a.
\end{enumerate}

\textbf{Direct Sufficiency and Necessity}:
\begin{enumerate}
    \item[] \textbf{AC1}: Satisfied since this is the observed state.
    \item[] \textbf{AC2a}: Satisfied for $M$ as the environment since $(M, \mathbf{u}) \vDash [E \leftarrow 500,U\leftarrow 500] Y \neq 500$ is true, as can be seen in the second row of the table. This satisfies AC2a because this counterexample demonstrates that there is a counterfactual assignment that is not sufficient. In this case, the User (U) is in the witness set, and matches.
    \item[] \textbf{AC2b}: Satisfied for $M$ as the environment since $(M, \mathbf{u}) \vDash [E \leftarrow 200,U\leftarrow 500] Y = 500$ is sufficient, and by putting $U$ in the witness set ensures that there are no other counterfactual assignments. 
    \item[] \textbf{AC3}: There is no strict subset of the engineer among which is an actual cause since that would imply the empty set, which does not satisfy AC2a.
\end{enumerate}

Thus the engineer is independently an actual cause of the output $U=500$ under both definitions. Now suppose that the output is offensive. This suggests that the Engineer is solely responsible for this output \textit{because they did not crash the chatbot}. This is a very strong and counterintuitive enforcement of responsibility. Furthermore, suppose that the engineer is an RL agent instead. If there was any chance of offensive output, the agent would always destroy the output since it is a sole actual cause, even though the user could have output any value and the chatbot parrots whatever the user inputs. 

This example demonstrates how the proximity of objects is not the key factor in assessing normality in actual causation. Instead, what matters is the size or likelihood of the event set that produces varying outputs. In this case, the engineer could have performed any number of actions other than crashing the bot, and would not have changed the result.
\section{Soft Necessity and Sufficiency}
\label{appendix: alpha-splitting}

This section describes the \textit{soft necessity and sufficiency} conditions used in Algorithm~\ref{exhaustive-search}. These conditions, denoted $\alpha_1$-necessity and $\alpha_0$-sufficiency, can be incorporated into FAC by replacing the contrastive necessity and weak sufficiency conditions. Unlike the invariant preimage and global minimality criterion discussed in the main paper, these properties only provide rough measures of normality through statistical requirements on the number of counterfactual outcomes. We plan to further investigate these properties in future work.

The intuition for $\alpha_1$-necessity and $\alpha_0$-sufficiency is as follows. By default, contrastive necessity requires there is only one alternative assignment of $\mathbf X$ which is not sufficient for the outcome. We instead require that at least $\alpha_1$ percent of the assignments of $\mathbf X$ are not sufficient for the outcome. Enforcing this would mean that for an event to be considered an actual cause, deviating from the observed event would need to cause a substantial number of possible effects. 

Similarly, weak sufficiency does not place hard requirements on the assignments of the non-causes, and any value of the non-causes can result in a change in outcome. However, we instead require that at least $(1-\alpha_0)$ percent of assignments of $\mathbf X$ produce the same outcome. In other words, the number of assignments of $\mathbf X$ for which the outcome is not invariant to the remaining variables does not exceed $\alpha_0 \cdot|\mathcal R(\mathbf X)|$. In general, these soft criteria turn out to be stricter than the Modified HP definition, though we later show that the sufficiency and necessity conditions from Modified HP can be recovered for specific values of the hyperparameters $\alpha_0$ and $\alpha_1$. We formalize these criteria below.

\begin{definition}[$\alpha_1$-necessity]
    For an observed state $\mathbf{s}^*$, the event $\mathbf{X} = \mathbf{x}$ is $\alpha_1$-necessary for the outcome $\mathbf{Y} = \mathbf{y}$ under the witness set $\mathbf{W}$ if the fraction of values within its support for which the resultant outcome deviates from the observed outcome $\mathbf{y}$ is at least $\alpha_1 \in [0,1]$.

    \begin{equation} \label{alpha1split}
        \frac{1}{|\mathcal{R}(\mathbf{X})|} \sum_{\mathbf{x^\prime} \in \mathcal{R}(\mathbf{X})} \mathbbm{1}((M, \mathbf{u}) \vDash [\mathbf{X} \leftarrow \mathbf{x^\prime}, \mathbf{W} \leftarrow \mathbf{s^*_W}] \mathbf{Y} \neq \mathbf{y}) \ge \alpha_1
    \end{equation}
\end{definition}

When $\alpha_1 = \frac{1}{|\mathcal{R}(\mathbf{X})|}$, the $\alpha_1$-necessity condition simplifies to the definition of contrastive necessity (Definition \ref{defn-necessity}).

$$\frac{1}{|\mathcal{R}(\mathbf{X})|} \sum_{\mathbf{x^\prime} \in \mathcal{R}(\mathbf{X})} \mathbbm{1}((M, \mathbf{u}) \vDash [\mathbf{X} \leftarrow \mathbf{x^\prime}, \mathbf{W} \leftarrow \mathbf{s^*_W}] \mathbf{Y} \neq \mathbf{y}) \ge \frac{1}{|\mathcal{R}(\mathbf{X})|}$$
$$\implies \sum_{\mathbf{x^\prime} \in \mathcal{R}(\mathbf{X})} \mathbbm{1}((M, \mathbf{u}) \vDash [\mathbf{X} \leftarrow \mathbf{x^\prime}, \mathbf{W} \leftarrow \mathbf{s^*_W}] \mathbf{Y} \neq \mathbf{y}) \ge 1$$
$$\implies \exists \mathbf{x^\prime} \in \mathcal{R}(\mathbf{X}) \quad (M, \mathbf{u}) \vDash [\mathbf{X} \leftarrow \mathbf{x^\prime}, \mathbf{W} \leftarrow \mathbf{s^*_W}] \mathbf{Y} \neq \mathbf{y}$$

By increasing the value of the ratio hyperparameter $\alpha_1$, the event is declared as necessary if a change in $\mathbf{X}$ results in a different outcome from the observed outcome in a greater percentage of states, instead of simply verifying that one such state exists. Doing so ensures that not only is the event necessary, but that a change in outcome is a more \textit{normal} or common occurrence, and not just an edge case.

\begin{definition}[$\alpha_0$-sufficiency]
    For an observed state $\mathbf{s}^*$, the event $\mathbf{X} = \mathbf{x}$ is $\alpha_0$-sufficient for the outcome $\mathbf{Y} = \mathbf{y}$ under the witness set $\mathbf{W}$ if the fraction of values within the support of its complement $\mathbf{C} = \mathbf{S} - (\mathbf{W} \cup \mathbf{X})$ for which the resultant outcome deviates from the observed outcome $\mathbf{y}$ is less than $ \alpha_0 \in [0,1]$.
    
    \begin{equation} \label{alpha0split}
         \frac{1}{|\mathcal{R}(\mathbf{C})|} \sum_{\mathbf{c} \in \mathcal{R}(\mathbf{C})} \mathbbm{1}((M, \mathbf{u}) \vDash [\mathbf{X} \leftarrow \mathbf{x}, \mathbf{C} \leftarrow \mathbf{c}, \mathbf{W} \leftarrow \mathbf{s^*_W}] \mathbf{Y} \neq \mathbf{y}) \le \alpha_0
    \end{equation}
    
\end{definition}

When $\alpha_0 = 0$, the $\alpha_0$-sufficiency condition simplifies to the definition of direct sufficiency (Definition \ref{defn-direct-sufficiency}). 

$$\frac{1}{|\mathcal{R}(\mathbf{C})|} \sum_{\mathbf{c} \in \mathcal{R}(\mathbf{C})} \mathbbm{1}((M, \mathbf{u}) \vDash [\mathbf{X} \leftarrow \mathbf{x}, \mathbf{C} \leftarrow \mathbf{c}, \mathbf{W} \leftarrow \mathbf{s^*_W}] \mathbf{Y} \neq \mathbf{y}) = 0$$
$$\implies (M, \mathbf{u}) \vDash [\mathbf{X} \leftarrow \mathbf{x}, \mathbf{C} \leftarrow \mathbf{c}, \mathbf{W} \leftarrow \mathbf{s^*_W}] \mathbf{Y} = \mathbf{y} \quad \forall \mathbf{c} \in \mathcal{R}(\mathbf{C})$$

When $\alpha_0 = 1 - \frac{1}{|\mathcal{R}(\mathbf{C})|}$, we get: 

$$\frac{1}{|\mathcal{R}(\mathbf{C})|} \sum_{\mathbf{c} \in \mathcal{R}(\mathbf{C})} \mathbbm{1}((M, \mathbf{u}) \vDash [\mathbf{X} \leftarrow \mathbf{x}, \mathbf{C} \leftarrow \mathbf{c}, \mathbf{W} \leftarrow \mathbf{s^*_W}] \mathbf{Y} \neq \mathbf{y}) \le 1 - \frac{1}{|\mathcal{R}(\mathbf{C})|}$$
$$\implies 1 - \frac{1}{|\mathcal{R}(\mathbf{C})|} \sum_{\mathbf{c} \in \mathcal{R}(\mathbf{C})} \mathbbm{1}((M, \mathbf{u}) \vDash [\mathbf{X} \leftarrow \mathbf{x}, \mathbf{C} \leftarrow \mathbf{c}, \mathbf{W} \leftarrow \mathbf{s^*_W}] \mathbf{Y} = \mathbf{y}) \le 1 - \frac{1}{|\mathcal{R}(\mathbf{C})|}$$
$$\implies \sum_{\mathbf{c} \in \mathcal{R}(\mathbf{C})} \mathbbm{1}((M, \mathbf{u}) \vDash [\mathbf{X} \leftarrow \mathbf{x}, \mathbf{C} \leftarrow \mathbf{c}, \mathbf{W} \leftarrow \mathbf{s^*_W}] \mathbf{Y} = \mathbf{y}) \ge 1$$
$$\implies \exists \mathbf{c} \in \mathcal{R}(\mathbf{C}) \quad (M, \mathbf{u}) \vDash [\mathbf{X} \leftarrow \mathbf{x}, \mathbf{C} \leftarrow \mathbf{c}, \mathbf{W} \leftarrow \mathbf{s^*_W}] \mathbf{Y} = \mathbf{y}$$

In other words, at least one state exists where the resulting outcome is $\mathbf{Y} = \mathbf{y}$. This trivially holds if weak sufficiency is satisfied since the observed state will result in outcome $\mathbf{Y} = \mathbf{y}$. For this setting of $\alpha_0$, weak sufficiency implies $\alpha_0$-sufficiency and $\alpha_0$-sufficiency also implies weak sufficiency unless \hyperref[ac1:fac]{AC1} is violated (which is the only case where the observed state $\mathbf{s^*}$ does not result in the outcome $\mathbf{Y} = \mathbf{y}$ but could be observed for some other state where $\mathbf{c} \neq \mathbf{s^*_C}$). From a practical standpoint, most algorithms for discovering actual causes will run an $O(1)$ check for AC1 and terminate immediately if it fails, so in practice we can assume that $\alpha_0$-sufficiency implies weak sufficiency for $\alpha_0 \ge 1 - \frac{1}{|\mathcal{R}(\mathbf{C})|}$.

In general, the motivation for applying soft necessity and sufficiency is to exclude spurious actual causes in small state spaces where the number of possible states is exceeded by the number of possible binary vectors. In these cases, states with equivalent outputs can be sequestered to a single IVP with a smaller corresponding binary vector. In other words, when the outcome space is small, a functional actual cause might return part of a cause instead of the full cause to reduce the cost of the overall assignment. For example, Table \ref{tab:no-alpha-sufficiency} shows the results in the 1D mover domain previously outlined in Section \ref{sec:algorithms}. The rows corresponding to binary $01$ show that when the obstacle obstructs the mover, it is possible to reduce the cost by assigning all states corresponding to the outcome of 1 to the binary $01$ such that only the obstacle position is an actual cause. We prevent this effect by forcing a certain percentage of states to take on certain values. In practice, we used $\alpha_0$-sufficiency but did not require $\alpha_1$-necessity to obtain accurate results.

\begin{table}[t]
    \centering
    \begin{minipage}{.48\textwidth}
      \centering
        \begin{tabular}{cccc}
        \toprule
        Binary& Mover& Obstacle& Outcome\\
        \midrule
        $01$& $0$& $0$& $1$\\
        $01$& $0$& $2$& $1$\\
        $01$& $1$& $1$& $2$\\
        $01$& $1$& $2$& $1$\\
        $10$& $0$& $1$& $0$\\
        $10$& $1$& $0$& $2$\\
        $10$& $2$& $0$& $3$\\
        $10$& $2$& $1$& $3$\\
        $10$& $2$& $2$& $3$\\
        $10$& $3$& $0$& $4$\\
        $10$& $3$& $1$& $4$\\
        $10$& $3$& $2$& $4$\\
        \bottomrule
        \end{tabular}
    \caption{1D Mover results with weak sufficiency instead of $\alpha_0$-sufficiency.}
    \label{tab:no-alpha-sufficiency}
    \end{minipage}%
    \hfill
    \begin{minipage}{.48\textwidth}
      \centering
	\begin{tabular}{cccc}
	\toprule
	Binary& Mover & Obstacle & Outcome\\
	\midrule
	$11$& $0$& $2$& $1$\\
	$01$& $1$& $2$& $1$\\
	$10$& $0$& $0$& $1$\\
	$10$& $0$& $1$& $0$\\
	$10$& $1$& $0$& $2$\\
	$10$& $1$& $1$& $2$\\
	$10$& $2$& $0$& $3$\\
	$10$& $2$& $1$& $3$\\
	$10$& $2$& $2$& $3$\\
	$10$& $3$& $0$& $4$\\
	$10$& $3$& $1$& $4$\\
	$10$& $3$& $2$& $4$\\
	\bottomrule
	\end{tabular}
     \caption{1D Mover results with $\alpha_0$-sufficiency, but without enforcing the invariance property (\textbf{P\ref{property:invariance}}).}
      \label{tab:alpha-sufficiency-no-invariance}
    \end{minipage}
\end{table}

However, $\alpha_0$-sufficiency and $\alpha_1$-necessity alone are not enough to represent normality, and the invariant preimage is still essential. Table \ref{tab:alpha-sufficiency-no-invariance} shows the results in the same 1D Mover example where $\alpha_0$-sufficiency is included but the invariance property (\textbf{P\ref{property:invariance}}) is not enforced. The states $[0,1]$ and $[1,1]$ have different outcomes but the binary vector $10$ is still assigned, implying that the outcome is not invariant to the obstacle position. For this reason, we set $\alpha_0 \neq 0$ to permit some counterfactual assignments of the candidate actual cause to result in different outcomes, provided they continue to satisfy the invariance property given by the invariant preimage containing the observed state.

It is worth noting that contrastive necessity and direct sufficiency do not readily scale to complex environments with continuous state spaces, since verifying them requires an enumeration of all counterfactual possibilities. In contrast, the invariance property can be approximated easily through masking as demonstrated in JACI.  
\section{Network and Training Details}

\label{AppendixNeuralModel}
We train two neural networks to identify actual causes from data, one to model the state-binary mapping $h_\mathbf{Y}(\mathbf s^{(i)};\theta)$, and the other to model the dynamics $f(\mathbf s^{(i)}, \mathbf b^{(i)};\phi)$. Both networks take in an unordered sequence of vectors corresponding to the causal variables; $f$ returns the state of the outcome variable and $h_\mathbf{Y}$ returns a binary vector corresponding to the inferred actual cause in that state. In practice, several details are used to ensure the networks are able to learn. 

Both networks utilize 1D convolutions to encode variable-wise invariance. A 1D convolution applies the same multi-layer perceptron (MLP) to each input, where each input is the state of a different causal variable. The first 1D convolution produces the embeddings. These embeddings are aggregated by adding together the outputs. This aggregated result allows the model to include global information and is appended to each of the inputs for a second 1D convolution. In $h_\mathbf{Y}$, the output of this second 1D convolution is $\mathbf b_i$ for each vector. In $f$, this vector is passed through a final MLP that computes the output. Without having two layers, global information would not be included in $h_\mathbf{Y}$, restricting the class of functions that $f$ can represent.

\begin{figure}
    \centering
    \begin{subfigure}[]{\label{BinaryNetwork}
        \includegraphics[width=0.51\textwidth]{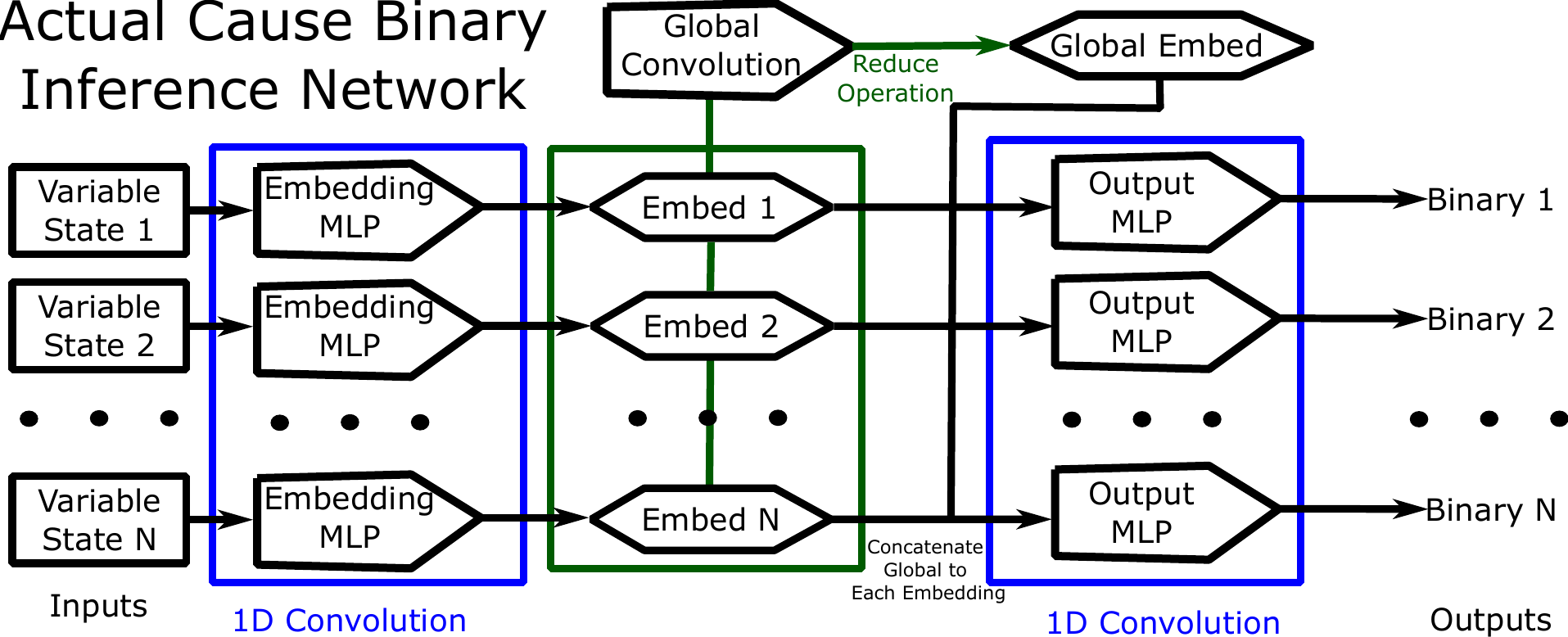}}
    \end{subfigure}
    \hfill
    \begin{subfigure}[]{\label{ForwardNetwork}
        \includegraphics[width=0.43\textwidth]{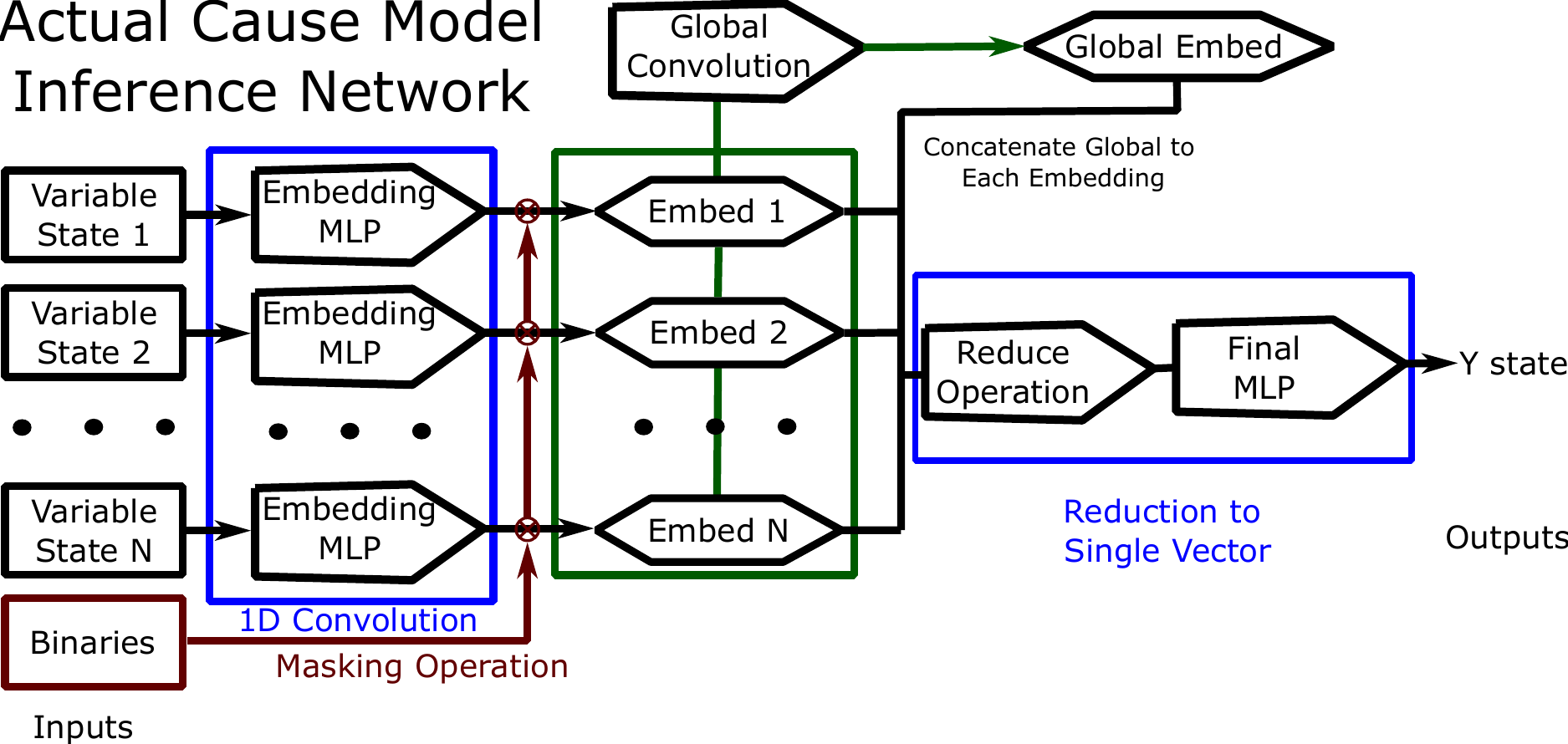}} 
    \end{subfigure}
    \caption{\textbf{(a)} Computational flow of $h_\mathbf{Y}(\mathbf s;\theta)$, which takes in the observed state for all causal variables and returns a binary vector indicating the actual cause. \textbf{(b)} Computational flow of $f(\mathbf s, \mathbf b;\phi)$, which takes in the observed state for all causal variables and a binary vector and returns the state of the outcome variable.}
\end{figure}

While the networks for $f$ and $h_\mathbf{Y}$ share the majority of their architecture, there are a few key differences. For $f$, a masking operation is applied after the first embedding operation: for any variable $k$ with $\mathbf b[k]=0$, the corresponding embedding of the state (the outcome of the 1D convolutions before reduction) is multiplied by $\mathbf b[k]$ to mask it out. Applying the masking operation at the embedding layer prevents the state of the variable $\mathbf s_k = \mathbf 0$ being confused with the binary vector mask $\mathbf b[k] = 0$. For $h_\mathbf{Y}$, the binary vectors are directly returned after the second convolution. 

The necessity property is encouraged through the structure of $h_\mathbf{Y}(\mathbf{s}^{(i)};\theta)$ constructs $\mathbf b$ by sharing the 1D convolution of $\mathbf b^{(i)}[k] = h_\mathbf{Y}(\mathbf{s}^{(i)}_k, \theta)$ for every $\mathbf{s}^{(i)}_k$. This structure biases the binary to be dependent on the input value, while still utilizing global information. Since the necessity property captures counterfactual dependence on a factor, this constraint biases towards necessity. We opt for this approach instead of explicitly generating counterfactual states and verifying necessity, since without a perfect model this would lead to instability during training. An illustration of the binary network is in Figure~\ref{BinaryNetwork}, and the masked forward model is in Figure~\ref{ForwardNetwork}. Hyperparameters are set identically for $f$ and $h_\mathbf{Y}$, and the 1D convolutions both use 3 hidden layers of size 256 and embed to size 512. 

In practice, $h_\mathbf{Y}$ cannot output hard $\{0,1\}$ binary values in training, since this would result in zero gradient flow from the forward model $f$ to the state-binary mapping results. However, if the forward model $f$ received stochastic inputs instead of values in $\{0,1\}$ it would affect the masking operation and result in poor prediction performance. We resolve this using a mixing distribution, where the network for $h_\mathbf{Y}$ outputs values in $\widehat{\mathbf b} \coloneqq [0,1]^{|\mathbf S|}$. We then use $\widehat{\mathbf{b}}$ as the probability of a Bernoulli random variable to give $\widetilde{\mathbf b} \coloneqq \text{Bern}(\widehat{\mathbf{b}})$, and input $\mathbf b\coloneqq \widetilde{\mathbf b} \cdot\widehat{\mathbf{b}}$ to the forward model $f$. This ensures gradient flow when training $h_\mathbf{Y}$, while still returning $\mathbf b_i \in \{0,1\}$.

\section{Random Vector Domain Details}
\label{appendix:random-vectors}

In this section, we elaborate on some of the details of the Random Vector domains, which are intended to represent the bipartite causal relationships that characterize Markov Decision Processes (MDPs) and RL. In such settings, the outcome is the result of a state transition $y\rightarrow y'$; we write $X_p$ in place of $y$ and $Y$ in place of $y^\prime$ to match the notation for candidate cause and outcome variables used in the rest of this paper. $X_p$ represents passive dynamics and is always a parent of $Y$. We collect episodes of length $50$, and at the end of each episode, the environment resets to a new initialization of random values between $[-1,1]$ for all variables. A causal variable in the domain is represented with a length $4$ vector clipped between $[-1,1]$. All variables, including the active variables, have passive dynamics, but for simplicity, we only represent passive dynamics for $Y$ in Figure \ref{RandomVectorGraphs}. For simplicity, we write the causal relationships as $\mathbf{A}x_a + \mathbf{B}x_p$, but in practice they are related by affine functions and there is no closed-form solution to solve for the outcome. Figure~\ref{RandomVectorGraphs} shows the connectivity of the different graphs evaluated in this work.

Aside from having extremely flexible graph connectivity, these domains are convenient since we can generate multiple versions of these graphs with the same connectivity, change the frequency of interactions by using $\mathbf{D}[x_a, x_p] > \tau$, change the dimensionality of the variables, and distinguish multiple independent interactions from jointly dependent interactions (as illustrated in \textbf{3-in} vs. \textbf{3-m-in}). Additionally, in these domains, there are none of the shortcuts such as spatial proximity that are sometimes used in other RL domains. This allows us to form a clear distinction between methods based on spatial proximity and methods based on context-specific independence, such as JACI. 

  \begin{figure}
\centering
  \centering
  \includegraphics[width=0.95\textwidth]{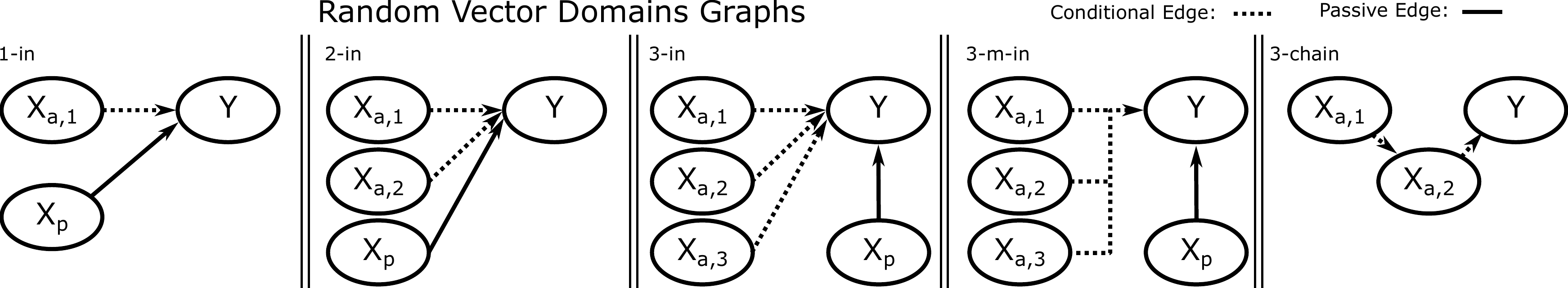}
\caption{Random Vectors domains, where dotted lines indicate conditional edges. In \textbf{3-m-in}, the three edges form a single connection such that the outcome varies according to all of the inputs. \textbf{$\tau$-1} and \textbf{d-20} use the same graph structure as \textbf{1-in}.
}
\label{RandomVectorGraphs}
\vskip -0.5cm
\end{figure}

When evaluating inference algorithms in these domains, we compute the accuracy using all the conditional edges. For example, in the \textbf{2-in} domain we take the accuracy of  $X_{a,1}\rightarrow Y$ and $ X_{a,2}\rightarrow Y$ combined. It is worth noting that not all edges exhibit exactly 50\% rates in the dataset, so guessing always 1 or always 0 for an actual cause might result in accuracy near, but not exactly equal to, $50\%$. In the same way, as the rate of appearance of conditional edges falls to  below $<5\%$, some errors might be cases where the outcome is truly indistinguishable from sampling ambiguity since the model is essentially recovering the decision boundary. It is also possible that the optimal model under Equation~\ref{optimization-learned-model} will incur some error to minimize the length of the binary vector, and these cases warrant further study as a possible limitation of JACI, or the Random Vectors domain.

The \textbf{3-chain} environment is characteristically different from the other environments in that it does not represent a dynamic relationship. As a result, it does not make sense to have a passive variable. As a result, the value of $Y$ will simply be $0$ in the absence of an interaction. However, the invariant relationships still hold in this case, so it is unclear why the performance is significantly worse. Furthermore, omitting the passive dynamics and using $0$ in the absence of an interaction tended to \textit{improve} results in the other domains, so this is likely not the cause for poor performance. We added the passive dynamics to make the domains \textit{more} challenging. We also assumed that by adding the relationship, the model would occasionally return $X_{a,1}$ instead of $X_{a,2}$ as a cause, but this did not happen. Additional experiments will be needed to determine why the error rate for the chain structure remains high. 

\section{JACI Training Details}
\label{JACITrainingDetails}

The primary hyperparameters that \algoacronym{} uses are the tradeoff rate $\beta$ between the minimality criterion for binary vectors and forward modeling loss, and the relative learning rate between the forward modeling network and the binary network that guides the joint optimization. Because the optimization algorithm is searching for a solution in a complex space with very strong local minima (such as $h_\mathbf{Y}$ always outputting $\mathbf{0}$ or $\mathbf{1}$), as well as strong constraints induced by the network structure, these hyperparameter values must be carefully chosen. For the \textbf{1-in} domain, Figure~\ref{AdaptiveRate} shows the rate of false positives and false negatives, and Figure~\ref{TrainCurve} shows the error rate during training.
 
Hand-tuning a fixed $\beta$ is often impossible and does not always result in recovering the true partitions. Instead, we use an adaptive $\beta$ based on the following insight: when $\lVert f(\mathbf{s}^{(i)}, b^{(i)}; \phi) - \mathbf{y}^{(i)}  \rVert$ is large, this indicates that the forward model is not performing well, and masking out additional information is counterproductive since the network probably does not have enough information to determine the outcome. This could be either because of masking or because it is early in training. Instead of a fixed tradeoff value, \algoacronym{} uses an adaptive tradeoff rate given by $\beta = \widehat{\beta} e^{-\lVert f(\mathbf{s}^{(i)}, b^{(i)}; \phi) - \mathbf{y}^{(i)}  \rVert}$, where $\widehat{\beta}$ is a hyperparameter. In practice, the inference accuracy of \algoacronym{} is somewhat robust to values of $\widehat{\beta}$, and we use a value between $1-10$ depending on the domain.

Additionally, the relative learning rates between the forward modeling and binary networks must be tuned to avoid local minima. If the binary network is trained too quickly, it may converge on behavior that only predicts based on the all 0 or all 1 binary, which is a local minima that the forward model cannot escape. Instead, the forward modeling and binary networks must be trained together so that each network can adapt to gradual changes in the other. In practice, this amounts to alternating once per training step between the forward modeling and binary networks, with the learning rate of the binary network initialized to $1/10$th that of the modeling network. However, further strategies for this alternation could yield better results. Other hyperparameters and their sensitivity are described in Table~\ref{Hyperparameters}.

\begin{figure}[t]
    \centering
    \begin{subfigure}[]{\label{AdaptiveRate}
        \includegraphics[width=0.48\textwidth]{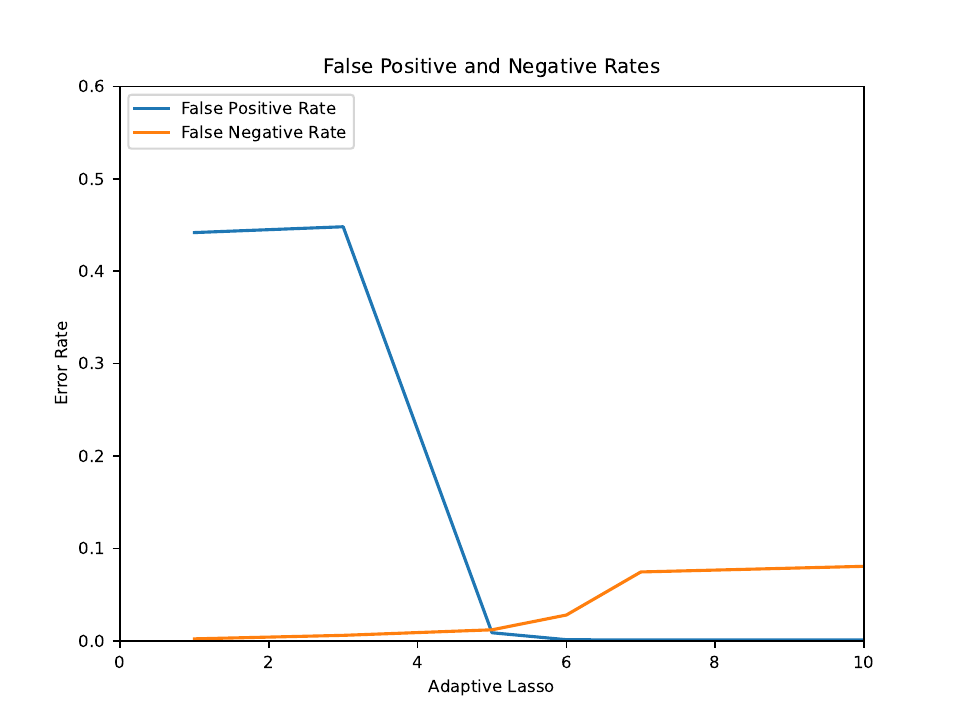}}
    \end{subfigure}
    \hfill
    \begin{subfigure}[]{\label{TrainCurve}
        \includegraphics[width=0.48\textwidth]{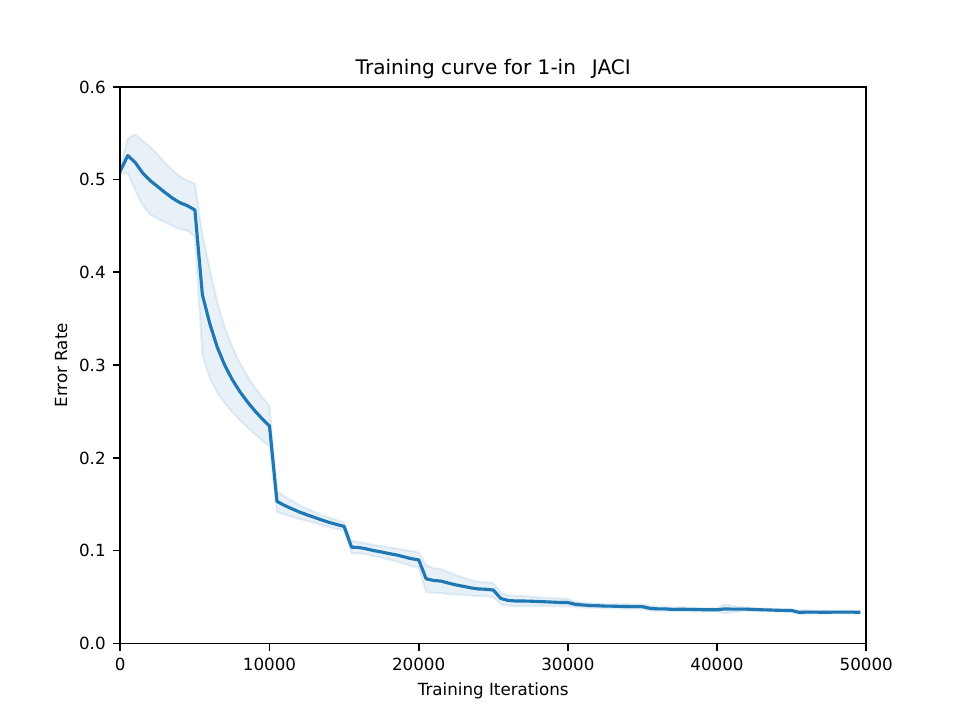}} 
    \end{subfigure}
    \caption{\textbf{(a)} Comparison of false positive and false negative rates in the \textbf{1-in} domain. The blue line indicates false positive rate (the model guesses an actual cause when there is none) and the orange line indicates false negative rate (the model guesses no actual cause when there is one). 
    \textbf{(b)} Train curve averaged over 10 runs for \textbf{1-in}. Error rate decreases monotonically, even though JACI optimizes Equation \ref{optimization-learned-model} instead of the error rate explicitly.}
\end{figure}

\begin{table*}[t]
\vskip 0.2cm
\begin{center}
\begin{small}
\begin{tabular}{@{}lcr@{}}
\toprule
\textbf{Parameter} & \textbf{Values} & \textbf{Sensitivity}\\
\midrule
Forward Model Learning rate & 0.001-0.003 & Significant if too high \\
Interaction Model Learning rate & 0.0001-0.0007 & Significant \\
Adaptive lasso rate & 1-10 & Significant, see Figure~\ref{AdaptiveRate}\\
Network layers & 6 & Fixed across all tasks \\
Network width & 256 & Fixed across all tasks \\
Embedding width & 1024 & Fixed across all tasks \\
Reduction Function & sum & Fixed across all tasks \\
Number of Frames & 100k, 500k  & Random Vectors or Breakout/Robot Pushing \\
Number of Training Iterations & 200k & Fixed across all tasks\\
Batch Size & 512 & Fixed across all tasks, sensitive if too low\\
Threshold for AC & 0.99 & Fixed across all tasks\\
\bottomrule
\end{tabular}
\caption{Hyperparameters}
\label{Hyperparameters}

\end{small}
\end{center}
\vskip -0.2in
\end{table*}
\section{Baseline Details}
\label{Baselines}
We compare against two baselines from the RL literature, as well as one baseline that enforces the necessity property based on counterfactual methods from the RL literature.

\subsection{Gradient Baseline}

This method is explored in some prior work~\citep{pitis2020counterfactual,wang2023elden}, and has shown promising results, especially in quasi-static domains where there is a significant change that occurs as a result of interactions. It takes the magnitude of the gradient of the inputs with respect to the output ($g_{i,\mathbf y}$) for each variable in a particular state, given by $g_{i,\mathbf y} = \left \| \frac{\delta \mathbf{y}}{\delta \mathbf{s}_i}\right \|_1$. In practice, when learning from observational data this gradient is of a learned forward model $f(\mathbf s;\phi)\rightarrow \mathbf y$. For a given threshold $\tau_g$, this method assigns an actual cause based on the magnitude of $g_{i,\mathbf y}$ for causal variable $i$. 
\begin{equation}
\mathbf{b}_i \coloneqq \mathbbm{1}\left(\left\|\frac{\delta \mathbf{y}}{\delta \mathbf{s}_i}\right\|_1 > \tau\right)
\label{gradient_baseline_eq}
\end{equation} 

However, tuning the magnitude of $\tau_g$ can be very challenging as it is domain and sometimes interaction-specific. This observation is corroborated in our results. In this work, we tune the threshold magnitude using oracle knowledge of the actual causes, by choosing the midpoint between the mean $g_{i,y}$ for the true actual causes and the mean $g_{i,y}$ for the true non-causes. This gives good results for cases with a single actual cause since $\tau_g$ can be perfectly tuned for the dynamics of that cause. However, the performance declines as the number of causes increases, since finding a threshold that holds for all of these dynamics becomes increasingly challenging. We experimented with regularized gradients as suggested in ~\citet{wang2023elden}, which reduces the size of the gradients by applying a cost function to the model $f(\mathbf s;\phi)\rightarrow \mathbf y$:
\begin{equation}
\mathcal L_\text{grad}(f, \mathbf s; \phi) \coloneqq \lambda_\text{ grad}\left\|\frac{\delta f(\mathbf s; \phi)}{\delta \mathbf{s}_i}\right\|_1
\label{gradient_regularization}
\end{equation} 

However, this regularization did not significantly improve results. This is not unexpected; the error rates we found are also comparable with those observed in these works. This baseline does have some additional advantages when compared with other baselines. While neural networks in high dimensions appear to overfit to some extent, as shown by the performance drop in the \textbf{d-20} environment in Table~\ref{AccuracyTable}), this dropoff is not as significant when using gradient-based inference. On top of this, the gradients can be computed without making any significant changes to the model or training procedure. However, the loss in performance makes its application limited for downstream tasks.
	
\subsection{Attention Baseline}
With the success of attention-based networks such as Transformers~\cite{vaswani2017attention}, the question naturally arises whether attention weights inherently capture actual causes. Prior work has used attention to identify relevant tokens~\cite{hafiz2021attention}. Using each of the causal variables as tokens, we can train an attention model and determine if the attention provides a meaningful signal of the actual causes. 

For this baseline, we implemented an attention network that uses multi-head attention over the input variables when predicting the outcome variable. For each attention head $a_i^k \in \{1,\hdots, H \}$, this method takes the average of the length $N$ vector where every index corresponds to one of the causal variables (which we represent with $a_i^k$ in abuse of notation). These are then incorporated in a multi-head attention strategy to infer the state of the outcome variable. To determine whether a variable is an actual cause, for each variable $X_{i}$ we compute 

\begin{equation}
\mathbf{b}_i \coloneqq \mathbbm{1}\left(\frac{\sum_{k=1}^Na_i^k}{N} > \tau\right)
\label{attention_baseline_eq}
\end{equation} 

If the ratio of heads $\tau$ is greater than $1/N$, that variable is considered an actual cause. This represents the attention weight on that particular input being greater than some amount. We performed some tuning to determine if a better threshold value for the ratio of heads could be determined but found this method to struggle regardless of what ratio was used. If we just used an ``at least one head'' instead of a soft measure, this methoud would infer every variable as an actual cause. This is analogous to the idea of $\alpha-$necessity introduced in Appendix~\ref{appendix: alpha-splitting}. In practice, this baseline still performs poorly because the heads rarely differentiate between the inputs enough to provide meaningful assessments of the actual cause. This lack of differentiation we observed could be due to training on only hundreds of thousands of tokens instead of trillions of tokens used in many other works using attention networks.

To improve the performance, we altered the training procedure of the attention model by introducing an entropy loss over the attention of the input on the heads, so that the model would learn to differentiate between the inputs at least to some extent. This loss has the form
\begin{equation}
\mathcal L_\text{attn}(a_i) \coloneqq -a_i\log(a_i)
\end{equation}
In practice, this prevented the network from converging to flat attention heads that put equal attention on all inputs.

\subsection{Counterfactual Baseline}

This baseline takes a generative approach to enforcing the necessity property and identifying the actual causes, and is corroborated in counterfactual-based techniques~\cite{pitis2020counterfactual,choi2023unsupervised}. This is done by first learning a forward model $f(\mathbf s; \phi) \rightarrow \mathbf y$. Then for each causal variable $X_i$, it compares the observed outcome with the output of the forward model for random interventions on $X_i$. We can represent the distribution of counterfactuals on variable $X_i$ in state $\mathbf s$ as $C_i(\mathbf s)$. Then the comparison is given by $E_{\mathbf v^\prime \sim C_i(\mathbf s)}\|f(\mathbf v';\phi) - y\|_1$. In practice, we use samples from a uniform distribution over the normalized state space of the causal variable. If the expected outcomes differ by more than a threshold $\tau_\text{CF}$, then the value is considered an actual cause. In essence, this is analogous to $\alpha$-necessity, where at least $\tau_\text{CF}$ differences should be observed to consider this a cause. Since the values of random vectors are constrained between $[-1,1]$, generating the intervention can be done by replacing the state of a variable $X_i$ with a length 4 random vector drawn from that range. The resulting binary vector is computed as:
\begin{equation}
    \mathbf{b}_i \coloneqq \mathbbm{1}\left(\frac{1}{N}\sum_{n=1}^N \| f(\mathbf{s}_i^{(n)}; \theta) \|_1 > \tau_\text{CF}\right)
\end{equation}

In practice, this method does not often get useful results and often overestimates actual causes. Part of this could be difficulty with determining a good value for $\tau_\text{CF}$, the threshold parameter since this value can fluctuate over different training seeds. Additionally, it is difficult to generate meaningful counterfactuals in domains like Breakout, where certain states are not actually possible. However, the primary flaw in this method is that networks trained with observational values can often struggle with counterfactual values due to overfitting. The poor performance of this method is also what dissuaded us from explicitly enforcing $\alpha-$necessity in \algoacronym{}.
\section{Unifying Local Cause, Interaction and other Existing Approximations of Actual Cause}
\label{Appendix:localcausality}
Several related concepts in the RL literature such as \textit{local causality} or \textit{interactions} explore various heuristics for identifying actual causes of a particular outcome $\mathbf Y = \mathbf y$. In these settings, the observed events are state transitions $\mathbf y\rightarrow \mathbf y'$ where $\mathbf y'$ is the next state of a temporal causal variable $\mathbf Y$. These heuristics often incorporate intuitions of normality, along with other intuitions such as necessity, sufficiency and minimality that are shared with prior definitions of actual cause. Since FAC extends these definitions by using the invariance property (\textbf{P\ref{property:invariance}}), to represent normality, a direct connection can be made between prior definitions of actual cause and heuristic-based definitions for RL. We highlight the relationships between these definitions based on the invariance, necessity and minimality properties.

\begin{table}[t]
\begin{center}
\begin{small}
\begin{tabular}{@{}lcccr@{}}
\toprule
 Method & Invariance & Necessity & Minimality \\
\midrule
FAC & General set partitioning & Arbitrary counterfactuals & Global minimality \\
Gradient & Local set partitioning & Local counterfactuals & Local regularization \\
Attention & General set partitioning & Not explicitly handled & Architecture regularization \\
Granger & Pairwise set partitioning & Pairwise counterfactuals & Not used \\
Physical & Heuristic set partitioning & Heuristic counterfactuals & Physics \\
\bottomrule
\end{tabular}
\caption{Summary of key properties of actual cause characterized by related methods.}
\label{FAClocalcomparison}
\end{small}
\end{center}
\end{table}

Gradients~\cite{wang2023elden} have been used to approximately infer actual causes, by defining them as values of causal variables with gradients larger than a threshold ($>\epsilon_\text{grad}$, Equation~\ref{gradient_baseline_eq}), where the input gradients come from a learned forward model $f_\phi: \mathbf S \rightarrow \mathbf S'$. This definition most closely characterizes \textit{local} causality, since it relies on locality assumptions about relative change to describe the components of actual cause. If we assume that the forward model perfectly captures the true dynamics, then the input gradients are an indicator of local variation. As long as the transitions are not at a critical point (second order gradient not equal to zero), switching to an adjacent state will have no effect on the transition. In the language of FAC, these adjacent states would be in the same invariant preimage. If the gradient is high there will be a alternative state in the local region that will produce a significantly different outcome, which approximates the necessity property. However, this does not hold for non-local counterfactual states. Furthermore, gradients do not explicitly capture minimality, although this can be approximated using a penalty term on the size of the gradients (Equation~\ref{gradient_regularization}) during training to prevent the model from being needlessly sensitive. 

Attention~\cite{pitis2020counterfactual} has also been used to approximately infer actual causes based on the magnitude of attention weights in a learned attention model (Equation~\ref{attention_baseline_eq}). While this removes some of the local dependence that gradients have, it also sacrifices some of the guarantees. Since inputs with zero attention will have no effect of the output, zero attention on an input mirrors the invariance property. However nonzero attention does not necessarily guarantee causal effect, meaning that the necessity property is not captured since a change in outcome is not guaranteed. While minimality is not directly captured, we found that adding a penalty to the attention heads resulted in improved inference accuracy in the attention baseline (Section~\ref{Baselines}. This suggests that incorporating minimality may be a useful tool for regularizing attention models. Besides these limitations, in practice many multilayer attention models learn \textit{flat attention} or equal weight over all inputs for some heads, which makes identifying a good threshold nearly impossible.

Granger-causal methods~\cite{chuck2023granger} identify actual causes by learning partial models and comparing their outputs at particular states. These methods learn a passive model $f:\mathbf Y\rightarrow \mathbf Y'$ that returns the next state given the present state alone, and an active model $g:\mathbf Y\times \mathbf X\rightarrow \mathbf Y'$ that returns the next state given the present state and action. When the outputs of both models are the same, this indicates that an interaction occurred. However, when the active model is more accurate, this indicates that an interaction did not occur. This approach models the invariance property with $f(\mathbf y)$ and necessity with $g(\mathbf y, \mathbf x)$. However,  since Granger causality methods only perform pairwise comparisons they are strictly more limited. 

Physical properties in the environment~\cite{chuck2020hypothesis,pitis2020counterfactual} have also been used to identify actual causes. These methods identify interactions by simply asserting several physical properties such as contact and quasistatic properties, or through changepoint analysis. In practice, physical properties inherently capture invariance and often also capture necessity by specifying the physical laws that govern when a causal interaction is produced. It can be argued that by capturing the minimal set of physical laws, these methods also capture minimality of actual causes. 

The invariance, necessity, and minimality properties in FAC offer a principled account of the fundamental assumptions being made by each of the above methods for identifying actual causes, and we describe the relationships between these assumptions and FAC in Table~\ref{FAClocalcomparison}. These heuristic-based methods are formulated in the context of MDPs and RL, where the causal graph is modeled by a bipartite graph with a strict temporal ordering. Preemption cases do not appear in such settings and are rarely taken into account, which is why there is no analogous concept to the witness set in these methods. Our empirical evaluation of JACI has reduced emphasis on preemption for this reason, though we note that FAC incorporates the witness set and the results in Appendix \ref{appendix-exhaustive} show that it handles preemption cases.

\section{2D Pushing Domain}
\label{2dpusher}
The causal variables in this domain are Action, Pusher, Block, Target, Obstacle1, Obstacle2, and Obstacle3, where Block is the outcome variable. Each variable is represented by a 3-dimensional continuous vector, where the first two dimensions encode the midpoint of the object, and the third dimension is used to indicate when the block is on the target; this value is fixed to 0 for non-target objects. Each of the objects occupies a 1x1 axis-aligned bounding box.

Actions in this domain move the pusher $[-1,1]^2$ units from its current position, as long as the pusher remains within the $[5,5]$ bounds. If the pusher collides with an obstacle, it will be stopped at the first edge of contact. Motion perpendicular to the obstacle will slide the pusher along its edge. If the pusher collides with the block, it will push the block along the first edge of contact. If the block is pushed into an obstacle, the block is stopped at the first edge of contact with the obstacle and the pushed is blocked along the edge of the contacting edge of the block in the axis of the obstacle. The block and pusher are not impeded by the target. The causal graph is illustrated in Figure~\ref{pushbreakcausal} (right).

\begin{figure}
\centering
  \centering
  \includegraphics[width=0.95\textwidth]{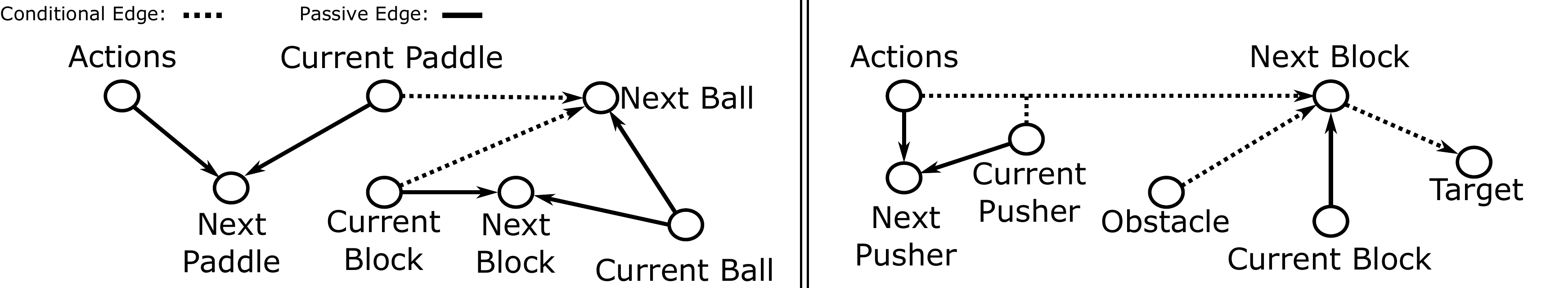}
\caption{Causal graphs for environment dynamics in Mini Breakout (left) and Robot Pushing (right), where transitions are from current positions to next positions. Note objects that do not move still exhibit dynamics (next position = current position), they are just omitted from the graph. A conditional edge is one that exhibits context-specific independence.
}
\label{pushbreakcausal}
\end{figure}

To generate data for this domain, the object positions are initialized to random, non-intersecting locations. In each episode, the pusher takes random actions for 20 timesteps, after which the positions of the obstacles, agent, target, and block are re-initialized to random locations. At each timestep, the current state of the block is an actual cause of its next state. However, the action and the pusher state are also actual causes if the action would have put the pusher in collision with the block. Similarly, the obstacle state is an actual cause if it impedes the block. This domain is challenging because the behavior of the block remains the same whether there are no actual causes or if the block is completely impeded by the obstacle. 

Since interactions are infrequent in this domain, the previous state of the block is the most frequent actual cause and other non-trivial actual causes are rare. By default, the 2D pusher interacts with the block less than $10\%$ of the time. To get meaningful statistics for the baselines, we resample states at test time such that non-trivial actual causes occur 50\% of the time. Since \algoacronym{} does not assume access to the true actual causes, in training we simulate the upsampling effect using a \textit{passive model}, which uses only the previous state of the block as input. Since we used the passive model for sampling, \algoacronym{} training never uses the true actual cause. We then upsample states where the passive model performs poorly. This method is inspired by~\citet{chuck2023granger}.

\section{Mini Breakout Domain}
\label{breakout}

The causal variables in this domain are Action, Paddle, Ball, Block1, Block2, and Block3, where Ball is the outcome variable. Each variable is represented with a 5-dimensional continuous vector where the first two dimensions encode midpoint position, the next two encode velocity, and the last dimension encodes additional information; for blocks, this value would indicate whether the block was hit or not. The size of the domain is 30x50, where the paddle is 7x3, the ball is 2x2 and the blocks are 5x2. 

Actions in this domain are discrete, where the paddle can move 2 units to the right or left, or remain stationary. The ball has a velocity in $\{-1,1\} \times \{-2,2\}$, and depending on what position it strikes the ball, it will receive an upward velocity of $(1,1), (1,2), (-1,2), (-1,1)$. To make the ball bounce independently of the current action, the angle is computed before the paddle moves. Bouncing off of the blocks will reverse the y velocity of the ball, and bouncing off of a wall will flip the ball velocity perpendicular to the wall. Otherwise, the ball will be updated by adding the velocity to the position every timestep. If a block is hit, the last dimension of its state vector will be set to 0 and it can no longer affect the ball. The causal graph for this domain is illustrated in Figure~\ref{pushbreakcausal} (left).

Data for this domain is generated by randomly initializing the ball with a y-velocity of $-1$, random x-position and x-velocity in $[-1,1]$, and y-position randomly in $[22, 28]$ (below the blocks). The blocks are initialized to random integer positions in $[30, 45]$, and the paddle is initialized at the position $(15,10)$. The environment resets if the ball hits the bottom or all the blocks are hit, and the agent takes random actions until one of the reset criteria is met. As in the 2D Robot Pushing domain, interactions occur less than $1\%$ of the time and non-trivial actual causes are rare. As a result, we enforce the same upsampling methods to get meaningful performance. At test time, we once again sample states such that non-trivial actual causes occur $\sim 50\%$ of the time.

\section{Connections to Reinforcement Learning}
\label{connections}

This work is motivated by applications of actual causality for Reinforcement Learning (RL) in complex domains. We reference some initial directions in Appendix~\ref{related-work} and describe possible future applications here. Prior work in actual causality has often neglected RL domains since the transition dynamics of RL systems are bipartite with strict temporal ordering under the Markov assumption. Preemption cases are less relevant in such settings, which limits the utility of existing definitions. However, actual causality offers several advantages for learning RL policies, such as enforcing normality properties for exploration, hierarchy, and learning effective world models. 

By knowing what caused a particular outcome at specific states, the RL agent can condition its behavior towards exploring these relations rather than simply over the whole state space. In domains such as 2D pushing, this can be the difference between flailing around in free space and interacting with the block. In higher dimensional domains, this flailing can be even more pronounced. This intuition follows readily into a hierarchy, where skills to control particular causes and effects can be reused to guide more complex policies. For example, learning to push a block can be used as a skill to learn to push around obstacles. Finally, RL agents have been shown to benefit from learning a world model that can be used to \textit{imagine}~\citep{hafner2019dream} rollouts. However, these world models can be brittle because the RL agent does not receive IID samples, but visits states conditioned on its behavior. These models can be made more robust by isolating the invariances that are indicated by actual causes. However, to realize these benefits, actual causes must be identifiable from data in a sample-efficient manner. We believe that FAC and \algoacronym{} are steps towards a formal definition for this process. Related work describing particular applications of these ideas are in Appendix~\ref{related-work}.

In addition to improving the process of learning a policy, actual causes can also be useful in analyzing learned policies. For example, actual causes can be used for policy explanation, by determining what events produced the observed agent behavior in a particular state. This is especially easy if actual causes were used when training the policy. In such applications, there are often actual causes that are undesirable (e.g. \textit{my episode failed because I crashed}), and others that show poor alignment (e.g. \textit{I caused injury to this person}), and inferring actual causes offers a way to check and enforce alignment. However, with the volume of high dimensional data processed by agents, prior actual cause inference methods would struggle to perform this kind of analysis, and there will be a large number of spurious causes if normality is not enforced as in FAC. On the other hand, \algoacronym{} can process high dimensional data and continue to perform well, while also providing accurate and sparse causes.

\section{Algorithm 0}
\label{appendix: algorithm-0}
We include a more in-depth description of Algorithm 0 in this section, which is equivalent to the one described in Section \ref{sec:algorithms}. While this algorithm simply enumerates all possible subset-binary pairs and tests them for the desired properties, its implementation can still be involved. 
To provide additional context to the 1D mover example in Section \ref{sec:algorithms}, 
 Figure~\ref{fig:results:1dmoverhist}b illustrates the distribution of valid partitions of the state for different states. While many valid binary-subset pairs exist, only a tiny subset of pairs minimize cost. These correspond to the few assignments of actual cause which capture intuitions such as normality and necessity.

\setcounter{algorithm}{-1}
\begin{algorithm}
\caption{Exhaustive Search}
\begin{algorithmic}
\label{exhaustive-long}
  \STATE {\bfseries Input:} SCM $M$, range function $\mathcal{R}$, ratios $(\alpha_0, \alpha_1)$ 
  \STATE {\bfseries Output:} Set of min-cost state-binary mappings $h_\mathbf{Y}$ and binary-subset pairs $(\mathcal{B}(\mathbf{Y}), \mathcal{I}(\mathbf{Y}))$
    \STATE \textbf{for} all states $\mathbf{s} \in \mathcal{R}(\mathbf{S})$:
    \STATE $\quad$ \textbf{let} $g_\mathbf{b}(\mathbf{s}) :=$ set of all binaries for which $\mathbf{s}$ satisfies necessity under any witness
    \STATE \textbf{for} all subsets of states $\mathcal{S} \in 2^{\mathcal{R}(\mathbf{S})}$:
    \STATE $\quad$ \textbf{let} $g_s(\mathcal{S}) := \bigcap_{\mathbf{s} \in \mathcal{S}} g_\mathbf{b}(\mathbf{s})$ \hfill $\triangleright$ Set of binaries for which all states in $\mathcal{S}$ satisfy necessity under any witness
    \STATE $\quad$ \textbf{remove} all binaries from $g_s(\mathcal{S})$ that do not satisfy invariance
    \STATE \textbf{let} $\mathcal{L}_{min} = \infty$ and $\mathcal{H} = \emptyset$
    \STATE \textbf{for} all possible partitionings of the state space $\mathcal{P} \in \mathcal{P}(\mathbf{S})$
    \STATE $\quad$ \textbf{for} all binary-partition pairings $(\mathcal{B}(\mathbf{Y}), \mathcal{I}(\mathbf{Y}))$ from $\mathcal{P}$ and $g_s(.)$
    \STATE $\quad \quad$ \textbf{if} all binaries in $\mathcal{B}(\mathbf{Y})$ are unique:
    \STATE $\quad \quad \quad$ \textbf{for} all $(\mathcal{I}_j, \mathbf{b}_j) \in (\mathcal{B}(\mathbf{Y}), \mathcal{I}(\mathbf{Y}))$ \textbf{let} $h_\mathbf{Y}(\mathbf{s}) := \mathbf{b}_j \; \forall \mathbf{s} \in \mathcal{I}_j$
    \hfill $\triangleright$ Get state-binary mapping
    \STATE $\quad \quad \quad$ \textbf{let} $\mathcal{L} := \sum_{j} |\mathbf{b}_j| |\mathcal{I}_j|$
    \STATE $\quad \quad \quad$ \textbf{if} $\mathcal{L} < \mathcal{L}_{min}$ \textbf{set} $\mathcal{L} = \mathcal{L}_{min}$ and $\mathcal{H} = \{ (h_\mathbf{Y}, \mathcal{B}(\mathbf{Y}), \mathcal{I}(\mathbf{Y})) \}$\hfill $\triangleright$ New min cost solution
    \STATE $\quad \quad \quad$ \textbf{else if} $\mathcal{L} = \mathcal{L}_{min}$ set $\mathcal{H} = \mathcal{H} \cup \{ (h_\mathbf{Y}, \mathcal{B}(\mathbf{Y}), \mathcal{I}(\mathbf{Y})) \}$\hfill $\triangleright$ Append to min cost solution set
    \STATE \textbf{return} set of solutions $\mathcal{H}$
\end{algorithmic}
\end{algorithm}
 
\begin{figure}[h]
\centering     
\subfigure[]{
\begin{tabular}[width=0.48\linewidth]{|l|c|c|c||c|c|c|r|}
\toprule
$\mathbf{b}$& $m$& $o$& $m^\prime$& $\mathbf{b}$& $m$& $o$& $m^\prime$\\
\midrule
$11$& $0$& $1$& $0$ & $10$& $2$& $0$& $3$\\
$11$& $1$& $2$& $1$ & $10$& $2$& $1$& $3$\\
$10$& $0$& $0$& $1$ & $10$& $2$& $2$& $3$\\
$10$& $0$& $2$& $1$ & $10$& $3$& $0$& $4$\\
$10$& $1$& $0$& $2$ & $10$& $3$& $1$& $4$\\
$10$& $1$& $1$& $2$ & $10$& $3$& $2$& $4$ \\
\bottomrule
\end{tabular}
}
\subfigure[]{\includegraphics[width=0.48\linewidth,valign=c]{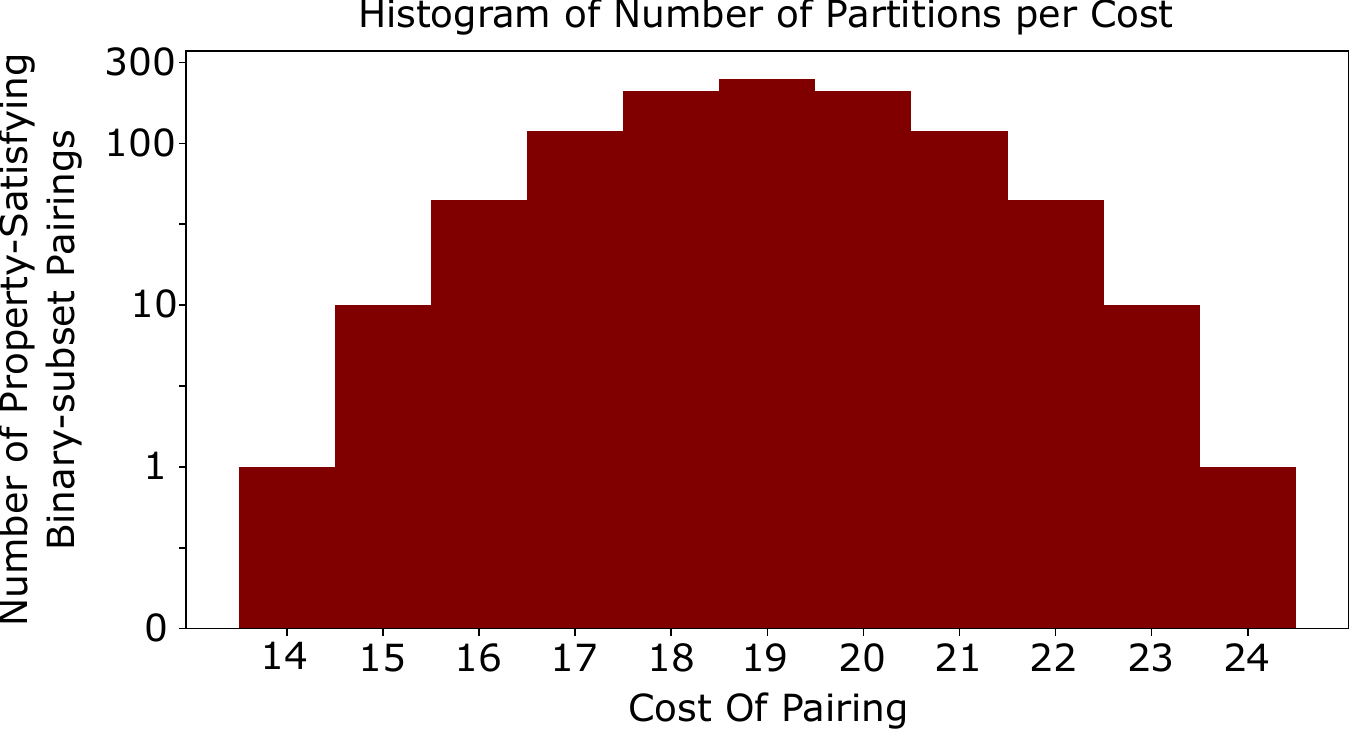}}
\caption{Results of Algorithm \ref{exhaustive-search} in 1D mover environment. Abbreviations: $\mathbf{b}$: binary [mover, obstacle], $m$: mover, $o$: obstacle, $m^\prime$: next mover position. 1D mover dynamics: $m^\prime = m + 1$ if $o \neq m + 1$, otherwise $m^\prime = m$. \textbf{(a)} Table of binary vector, state and outcome for the mover environment. The obstacle is only identified as a cause when it blocks the mover. \textbf{(b)}: Histogram of costs and number of valid binary-subset pairs.
}
\vskip -0.5cm
\label{fig:results:1dmoverhist}
\end{figure}
 
\section{Results of Algorithm \ref{exhaustive-search} on AC Examples}
\label{appendix-exhaustive}

In this section, we compare the results of Algorithm \ref{exhaustive-search} with existing verdicts of actual causation in the literature. In prior work on actual causality, proposed definitions are usually compared by assessing their verdicts on a set of examples from the philosophy literature, for which there are commonly accepted answers to serve as ground truth. The examples we consider are:

\begin{enumerate}
    \item Forest Fire (Example 7.1.4, \citet{halpern2016actual})
    \item Gang Execution (Section 4, \citet{rosenberg2018review})
    \item Halt and Charge (Example 2.3.7, \citet{halpern2016actual})
    \item Switching Railroad Tracks (Example 2.3.6, \citet{halpern2016actual})
    \item Binary AND ($C = A \land B$)
    \item Binary OR ($C = A \lor B$)
    \item Binary XOR ($C = A \oplus B$)
    \item Rock Throwing (Example 2.3.3, \citet{halpern2016actual})
\end{enumerate}

Unlike existing methods, we evaluate not only the actual causes of a particular state and outcome but also the actual causes for every combination of states and outcomes. This is by construction, since verifying that the functional actual cause definition (Definition \ref{fac-defn}) is satisfied requires examining the full state space and constructing a binary-subset pair. Each table corresponds to a minimal binary-subset pair, and each row indicates the inferred actual cause for a particular state and outcome. To compare against existing definitions, we consider the observed state that corresponds to what happened in the example from the literature and look at its corresponding row in the tables to see if the agreed-upon answer matches with the binary vector returned by Algorithm \ref{exhaustive-search}. 

Note that because Algorithm 0 returns \textit{all} minimal cost subset-binary pairings, there may be multiple tables for a single example. This ambiguity appears to be inherent to the actual cause problem: If there are multiple ways of arriving at a dependency, then any one way could be an actual cause, though the combination of all of them would not be a minimal actual cause. This property is not unique to FAC---prior definitions of actual cause can also infer multiple minimal actual causes of a particular event. The implications of this ambiguity are particularly pronounced in examples like Rock Throwing, and we can see this reflected in the number of tables returned.

\subsection{Forest Fire}

The binary variables $April$, $May$, and $June$ indicate whether or not there was rain in April, electrical storms in May, and electrical storms in June. The outcome $Fire = 1$ indicates a forest fire in May, and $Fire = 2$ indicates a forest fire in June. All tables indicate that when there was rain in April as well as electrical storms in May and June, the storm in May was not an actual cause. However, the rain in April was not an actual cause on its own, and the storms in June are also included in the actual cause.

\begin{multicols}{2}
\begin{tabular}{c|cccc}
\toprule
Binary& April& May& June& Fire\\
\midrule
$011$& $0$& $0$& $0$& $0$\\
$011$& $1$& $0$& $0$& $0$\\
$011$& $0$& $0$& $1$& $2$\\
$110$& $0$& $1$& $0$& $1$\\
$110$& $0$& $1$& $1$& $1$\\
$101$& $1$& $1$& $0$& $0$\\
$101$& $1$& $0$& $1$& $2$\\
$101$& $1$& $1$& $1$& $2$\\
\bottomrule
\end{tabular}

\begin{tabular}{c|cccc}
\toprule
Binary& April& May& June& Fire\\
\midrule
$011$& $0$& $0$& $0$& $0$\\
$011$& $0$& $0$& $1$& $2$\\
$110$& $0$& $1$& $0$& $1$\\
$110$& $0$& $1$& $1$& $1$\\
$101$& $1$& $0$& $0$& $0$\\
$101$& $1$& $1$& $0$& $0$\\
$101$& $1$& $0$& $1$& $2$\\
$101$& $1$& $1$& $1$& $2$\\
\bottomrule
\end{tabular}

\begin{tabular}{c|cccc}
\toprule
Binary& April& May& June& Fire\\
\midrule
$110$& $0$& $1$& $0$& $1$\\
$110$& $0$& $1$& $1$& $1$\\
$011$& $0$& $0$& $0$& $0$\\
$011$& $1$& $0$& $0$& $0$\\
$011$& $0$& $0$& $1$& $2$\\
$011$& $1$& $0$& $1$& $2$\\
$101$& $1$& $1$& $0$& $0$\\
$101$& $1$& $1$& $1$& $2$\\
\bottomrule
\end{tabular}

\begin{tabular}{c|cccc}
\toprule
Binary& April& May& June& Fire\\
\midrule
$110$& $0$& $1$& $0$& $1$\\
$110$& $0$& $1$& $1$& $1$\\
$011$& $0$& $0$& $0$& $0$\\
$011$& $0$& $0$& $1$& $2$\\
$011$& $1$& $0$& $1$& $2$\\
$101$& $1$& $0$& $0$& $0$\\
$101$& $1$& $1$& $0$& $0$\\
$101$& $1$& $1$& $1$& $2$\\
\bottomrule
\end{tabular}

\end{multicols}

\subsection{Gang Execution}

The binary variables $Gang$ and $Leader$ indicate whether or not the gang member or the gang leader respectively shot the victim, and the outcome variable $Death$ indicates whether the victim died. In all cases, the gang leader is declared as the actual cause, which is in agreement with the conclusions of \citet{rosenberg2018review}.

\begin{tabular}{c|ccc}
\toprule
Binary& Gang& Leader& Death\\
\midrule
$01$& $0$& $0$& $0$\\
$01$& $1$& $1$& $1$\\
\bottomrule
\end{tabular}

\subsection{Halt and Charge}

In this example, orders from higher-ranking officers could trump those of lower-ranking officers. The $Major$ can order the corporal to halt or charge (these are indicated by values 0 and 1 respectively), or choose not to make an order (indicated by a value of 2). If the $Major$ makes no order, the corporal will defer to the order of the $Sergeant$, where 1 indicates an order to charge and 0 indicates an order to halt. The outcome $Corporal = 1$ indicates that the corporal charged. 

The results show that any time the major gives the order to halt or charge, the major is the sole actual cause since its order trumps the sergeant. However, when the major chooses to abstain and the corporal follows the order of the sergeant, the actions of the major and the sergeant are jointly responsible for the outcome.

\begin{tabular}{c|ccc}
\toprule
Binary& Major& Sergeant& Corporal\\
\midrule
$10$& $0$& $1$& $0$\\
$10$& $1$& $0$& $1$\\
$10$& $0$& $0$& $0$\\
$10$& $1$& $1$& $1$\\
$11$& $2$& $0$& $0$\\
$11$& $2$& $1$& $1$\\
\bottomrule
\end{tabular}

\subsection{Binary And}

The outcome is zero when either A or B is zero. When only one of them is zero, the variable that is equal to zero is declared as the actual cause. When both of them are zero, they are jointly considered the actual cause, since they both have to change for the outcome to change to 1. 

The outcome is equal to 1 only when both A and B are equal to 1. In this case, we see that there are two minimal binary-subset pairs; one declares A as the actual cause and the other declares B as the actual cause. This reflects the ambiguity over which of the two is responsible when the outcome is determined by both jointly. Prior definitions of actual cause would also declare either of them to be actual causes by putting the other in the witness set.

\begin{multicols}{2}

\begin{tabular}{c|ccc}
\toprule
Binary& A& B& C\\
\midrule
$11$& $0$& $0$& $0$\\
$10$& $0$& $1$& $0$\\
$01$& $1$& $0$& $0$\\
$01$& $1$& $1$& $1$\\
\bottomrule
\end{tabular}

\begin{tabular}{c|ccc}
\toprule
Binary& A& B& C\\
\midrule
$11$& $0$& $0$& $0$\\
$01$& $1$& $0$& $0$\\
$10$& $0$& $1$& $0$\\
$10$& $1$& $1$& $1$\\
\bottomrule
\end{tabular}
\end{multicols}

\subsection{Binary Or}

The outcome is equal to zero only when both A and B are equal to zero. Similar to Binary AND, we have two minimal binary-subset pairs; one declares A as the actual cause and the other declares B as the actual cause. 

\begin{multicols}{2}

\begin{tabular}{c|ccc}
\toprule
Binary& A& B& C\\
\midrule
$01$& $0$& $0$& $0$\\
$01$& $0$& $1$& $1$\\
$10$& $1$& $0$& $1$\\
$11$& $1$& $1$& $1$\\
\bottomrule
\end{tabular}

\begin{tabular}{c|ccc}
\toprule
Binary& A& B& C\\
\midrule
$01$& $0$& $1$& $1$\\
$10$& $0$& $0$& $0$\\
$10$& $1$& $0$& $1$\\
$11$& $1$& $1$& $1$\\
\bottomrule
\end{tabular}
\end{multicols}

\subsection{Binary XOR}

The outcome is 1 when only one of A or B is equal to 1, and the other is zero. When both A and B have the same value, the outcome is zero. In all cases, both A and B are individually minimal actual causes of the outcome, and similar to Binary AND and OR we have two minimal binary-subset pairs such that either A or B is the actual cause.

\begin{multicols}{2}
\begin{tabular}{c|ccc}
\toprule
Binary& A& B& C\\
\midrule
$01$& $0$& $1$& $1$\\
$01$& $1$& $0$& $1$\\
$10$& $0$& $0$& $0$\\
$10$& $1$& $1$& $0$\\
\bottomrule
\end{tabular}

\begin{tabular}{c|ccc}
\toprule
Binary& A& B& C\\
\midrule
$10$& $0$& $1$& $1$\\
$10$& $1$& $0$& $1$\\
$01$& $0$& $0$& $0$\\
$01$& $1$& $1$& $0$\\
\bottomrule
\end{tabular}
\end{multicols}

\subsection{Rock Throwing}
The variables in the model are Suzy Hits (SH), Suzy Throws (ST), Billy Hits (BH), Billy Throws (BT), and Bottle Shatters (BoS). The results show that whenever Suzy throws, her throw is always the sole actual cause whether or not Billy throws. However, the equations in the original example are such that whenever Suzy throws her rock, it will always hit the bottle (i.e. SH = ST). As a result, there is an ambiguity between $SH = 1$ and $ST = 1$ as candidate actual causes. Our method returns all minimal-binary subset pairs, which encompasses both of these valid answers. Similarly, when Billy is the only one who throws and his rock hits the bottle, both $BH = 1$ and $BT = 1$ are candidate actual causes.

However, when neither of them throws and the bottle does not shatter, any of $ST = 0, SH = 0, BT = 0$ and $BH = 0$ could be actual causes, since any one of them throwing a rock at the bottle would have caused it to shatter. Since there are more ambiguities and valid answers in this example, the number of minimal binary-subset pairs increases, and this is reflected in the number of tables returned below. 

\begin{multicols}{2}
\begin{tabular}{c|ccccc}
\toprule
Binary& SH& ST& BH& BT& BoS\\
\midrule
$0010$& $0$& $0$& $0$& $0$& $0$\\
$0010$& $0$& $0$& $1$& $1$& $1$\\
$0100$& $1$& $1$& $0$& $0$& $1$\\
$0100$& $1$& $1$& $0$& $1$& $1$\\
\bottomrule
\end{tabular}

\begin{tabular}{c|ccccc}
\toprule
Binary& SH& ST& BH& BT& BoS\\
\midrule
$0010$& $0$& $0$& $0$& $0$& $0$\\
$0010$& $0$& $0$& $1$& $1$& $1$\\
$1000$& $1$& $1$& $0$& $0$& $1$\\
$1000$& $1$& $1$& $0$& $1$& $1$\\
\bottomrule
\end{tabular}

\begin{tabular}{c|ccccc}
\toprule
Binary& SH& ST& BH& BT& BoS\\
\midrule
$0001$& $0$& $0$& $0$& $0$& $0$\\
$0001$& $0$& $0$& $1$& $1$& $1$\\
$0100$& $1$& $1$& $0$& $0$& $1$\\
$0100$& $1$& $1$& $0$& $1$& $1$\\
\bottomrule
\end{tabular}

\begin{tabular}{c|ccccc}
\toprule
Binary& SH& ST& BH& BT& BoS\\
\midrule
$0001$& $0$& $0$& $0$& $0$& $0$\\
$0001$& $0$& $0$& $1$& $1$& $1$\\
$1000$& $1$& $1$& $0$& $0$& $1$\\
$1000$& $1$& $1$& $0$& $1$& $1$\\
\bottomrule
\end{tabular}

\begin{tabular}{c|ccccc}
\toprule
Binary& SH& ST& BH& BT& BoS\\
\midrule
$0010$& $0$& $0$& $1$& $1$& $1$\\
$0100$& $0$& $0$& $0$& $0$& $0$\\
$0100$& $1$& $1$& $0$& $0$& $1$\\
$0100$& $1$& $1$& $0$& $1$& $1$\\
\bottomrule
\end{tabular}

\begin{tabular}{c|ccccc}
\toprule
Binary& SH& ST& BH& BT& BoS\\
\midrule
$0010$& $0$& $0$& $1$& $1$& $1$\\
$1000$& $0$& $0$& $0$& $0$& $0$\\
$1000$& $1$& $1$& $0$& $0$& $1$\\
$1000$& $1$& $1$& $0$& $1$& $1$\\
\bottomrule
\end{tabular}

\begin{tabular}{c|ccccc}
\toprule
Binary& SH& ST& BH& BT& BoS\\
\midrule
$0001$& $0$& $0$& $1$& $1$& $1$\\
$0100$& $0$& $0$& $0$& $0$& $0$\\
$0100$& $1$& $1$& $0$& $0$& $1$\\
$0100$& $1$& $1$& $0$& $1$& $1$\\
\bottomrule
\end{tabular}

\begin{tabular}{c|ccccc}
\toprule
Binary& SH& ST& BH& BT& BoS\\
\midrule
$0001$& $0$& $0$& $1$& $1$& $1$\\
$1000$& $0$& $0$& $0$& $0$& $0$\\
$1000$& $1$& $1$& $0$& $0$& $1$\\
$1000$& $1$& $1$& $0$& $1$& $1$\\
\bottomrule
\end{tabular}

\begin{tabular}{c|ccccc}
\toprule
Binary& SH& ST& BH& BT& BoS\\
\midrule
$0100$& $0$& $0$& $0$& $0$& $0$\\
$0010$& $0$& $0$& $1$& $1$& $1$\\
$1000$& $1$& $1$& $0$& $0$& $1$\\
$1000$& $1$& $1$& $0$& $1$& $1$\\
\bottomrule
\end{tabular}

\begin{tabular}{c|ccccc}
\toprule
Binary& SH& ST& BH& BT& BoS\\
\midrule
$0100$& $0$& $0$& $0$& $0$& $0$\\
$0001$& $0$& $0$& $1$& $1$& $1$\\
$1000$& $1$& $1$& $0$& $0$& $1$\\
$1000$& $1$& $1$& $0$& $1$& $1$\\
\bottomrule
\end{tabular}

\begin{tabular}{c|ccccc}
\toprule
Binary& SH& ST& BH& BT& BoS\\
\midrule
$1000$& $0$& $0$& $0$& $0$& $0$\\
$0010$& $0$& $0$& $1$& $1$& $1$\\
$0100$& $1$& $1$& $0$& $0$& $1$\\
$0100$& $1$& $1$& $0$& $1$& $1$\\
\bottomrule
\end{tabular}

\begin{tabular}{c|ccccc}
\toprule
Binary& SH& ST& BH& BT& BoS\\
\midrule
$1000$& $0$& $0$& $0$& $0$& $0$\\
$0001$& $0$& $0$& $1$& $1$& $1$\\
$0100$& $1$& $1$& $0$& $0$& $1$\\
$0100$& $1$& $1$& $0$& $1$& $1$\\
\bottomrule
\end{tabular}

\begin{tabular}{c|ccccc}
\toprule
Binary& SH& ST& BH& BT& BoS\\
\midrule
$0010$& $0$& $0$& $0$& $0$& $0$\\
$0001$& $0$& $0$& $1$& $1$& $1$\\
$0100$& $1$& $1$& $0$& $0$& $1$\\
$0100$& $1$& $1$& $0$& $1$& $1$\\
\bottomrule
\end{tabular}

\begin{tabular}{c|ccccc}
\toprule
Binary& SH& ST& BH& BT& BoS\\
\midrule
$0010$& $0$& $0$& $0$& $0$& $0$\\
$0001$& $0$& $0$& $1$& $1$& $1$\\
$1000$& $1$& $1$& $0$& $0$& $1$\\
$1000$& $1$& $1$& $0$& $1$& $1$\\
\bottomrule
\end{tabular}

\begin{tabular}{c|ccccc}
\toprule
Binary& SH& ST& BH& BT& BoS\\
\midrule
$0001$& $0$& $0$& $0$& $0$& $0$\\
$0010$& $0$& $0$& $1$& $1$& $1$\\
$0100$& $1$& $1$& $0$& $0$& $1$\\
$0100$& $1$& $1$& $0$& $1$& $1$\\
\bottomrule
\end{tabular}

\begin{tabular}{c|ccccc}
\toprule
Binary& SH& ST& BH& BT& BoS\\
\midrule
$0001$& $0$& $0$& $0$& $0$& $0$\\
$0010$& $0$& $0$& $1$& $1$& $1$\\
$1000$& $1$& $1$& $0$& $0$& $1$\\
$1000$& $1$& $1$& $0$& $1$& $1$\\
\bottomrule
\end{tabular}

\begin{tabular}{c|ccccc}
\toprule
Binary& SH& ST& BH& BT& BoS\\
\midrule
$0100$& $0$& $0$& $0$& $0$& $0$\\
$0100$& $1$& $1$& $0$& $0$& $1$\\
$0010$& $0$& $0$& $1$& $1$& $1$\\
$1000$& $1$& $1$& $0$& $1$& $1$\\
\bottomrule
\end{tabular}

\begin{tabular}{c|ccccc}
\toprule
Binary& SH& ST& BH& BT& BoS\\
\midrule
$0100$& $0$& $0$& $0$& $0$& $0$\\
$0100$& $1$& $1$& $0$& $0$& $1$\\
$0001$& $0$& $0$& $1$& $1$& $1$\\
$1000$& $1$& $1$& $0$& $1$& $1$\\
\bottomrule
\end{tabular}

\begin{tabular}{c|ccccc}
\toprule
Binary& SH& ST& BH& BT& BoS\\
\midrule
$1000$& $0$& $0$& $0$& $0$& $0$\\
$1000$& $1$& $1$& $0$& $0$& $1$\\
$0010$& $0$& $0$& $1$& $1$& $1$\\
$0100$& $1$& $1$& $0$& $1$& $1$\\
\bottomrule
\end{tabular}

\begin{tabular}{c|ccccc}
\toprule
Binary& SH& ST& BH& BT& BoS\\
\midrule
$1000$& $0$& $0$& $0$& $0$& $0$\\
$1000$& $1$& $1$& $0$& $0$& $1$\\
$0001$& $0$& $0$& $1$& $1$& $1$\\
$0100$& $1$& $1$& $0$& $1$& $1$\\
\bottomrule
\end{tabular}

\begin{tabular}{c|ccccc}
\toprule
Binary& SH& ST& BH& BT& BoS\\
\midrule
$0100$& $1$& $1$& $0$& $0$& $1$\\
$0010$& $0$& $0$& $0$& $0$& $0$\\
$0010$& $0$& $0$& $1$& $1$& $1$\\
$1000$& $1$& $1$& $0$& $1$& $1$\\
\bottomrule
\end{tabular}

\begin{tabular}{c|ccccc}
\toprule
Binary& SH& ST& BH& BT& BoS\\
\midrule
$0100$& $1$& $1$& $0$& $0$& $1$\\
$0001$& $0$& $0$& $0$& $0$& $0$\\
$0001$& $0$& $0$& $1$& $1$& $1$\\
$1000$& $1$& $1$& $0$& $1$& $1$\\
\bottomrule
\end{tabular}

\begin{tabular}{c|ccccc}
\toprule
Binary& SH& ST& BH& BT& BoS\\
\midrule
$1000$& $1$& $1$& $0$& $0$& $1$\\
$0010$& $0$& $0$& $0$& $0$& $0$\\
$0010$& $0$& $0$& $1$& $1$& $1$\\
$0100$& $1$& $1$& $0$& $1$& $1$\\
\bottomrule
\end{tabular}

\begin{tabular}{c|ccccc}
\toprule
Binary& SH& ST& BH& BT& BoS\\
\midrule
$1000$& $1$& $1$& $0$& $0$& $1$\\
$0001$& $0$& $0$& $0$& $0$& $0$\\
$0001$& $0$& $0$& $1$& $1$& $1$\\
$0100$& $1$& $1$& $0$& $1$& $1$\\
\bottomrule
\end{tabular}

\begin{tabular}{c|ccccc}
\toprule
Binary& SH& ST& BH& BT& BoS\\
\midrule
$0100$& $1$& $1$& $0$& $0$& $1$\\
$0010$& $0$& $0$& $1$& $1$& $1$\\
$1000$& $0$& $0$& $0$& $0$& $0$\\
$1000$& $1$& $1$& $0$& $1$& $1$\\
\bottomrule
\end{tabular}

\begin{tabular}{c|ccccc}
\toprule
Binary& SH& ST& BH& BT& BoS\\
\midrule
$0100$& $1$& $1$& $0$& $0$& $1$\\
$0001$& $0$& $0$& $1$& $1$& $1$\\
$1000$& $0$& $0$& $0$& $0$& $0$\\
$1000$& $1$& $1$& $0$& $1$& $1$\\
\bottomrule
\end{tabular}

\begin{tabular}{c|ccccc}
\toprule
Binary& SH& ST& BH& BT& BoS\\
\midrule
$1000$& $1$& $1$& $0$& $0$& $1$\\
$0010$& $0$& $0$& $1$& $1$& $1$\\
$0100$& $0$& $0$& $0$& $0$& $0$\\
$0100$& $1$& $1$& $0$& $1$& $1$\\
\bottomrule
\end{tabular}

\begin{tabular}{c|ccccc}
\toprule
Binary& SH& ST& BH& BT& BoS\\
\midrule
$1000$& $1$& $1$& $0$& $0$& $1$\\
$0001$& $0$& $0$& $1$& $1$& $1$\\
$0100$& $0$& $0$& $0$& $0$& $0$\\
$0100$& $1$& $1$& $0$& $1$& $1$\\
\bottomrule
\end{tabular}

\end{multicols}

\subsection{Switching Railroad Tracks}

The model in the original example asserts that the train will reach its destination if there is no breakdown ($T \neq 2)$, but there is no affordance in the model for breakdowns to actually occur. For this reason, we add another binary variable $Break$ to the model and set $Track$ to 2 if $Break = 1$. 

The results show that whenever there is a breakdown, both $Break = 1$ and $Track = 2$ are returned as valid actual causes since as shown above, we always have $Track = 2$ when $Break = 1$. When there is no breakdown, the train reaches the station regardless of which track it arrived on, so as before we have the assignments to $Break$ and $Track$ as valid actual causes. In all tables below, the switch is never an actual cause.

\begin{multicols}{2}
\begin{tabular}{c|cccc}
\toprule
Binary& Break& Switch& Track& Arrive\\
\midrule
$100$& $0$& $0$& $0$& $1$\\
$100$& $0$& $1$& $1$& $1$\\
$100$& $1$& $0$& $2$& $0$\\
$100$& $1$& $1$& $2$& $0$\\
\bottomrule
\end{tabular}

\begin{tabular}{c|cccc}
\toprule
Binary& Break& Switch& Track& Arrive\\
\midrule
$001$& $0$& $0$& $0$& $1$\\
$001$& $0$& $1$& $1$& $1$\\
$001$& $1$& $0$& $2$& $0$\\
$001$& $1$& $1$& $2$& $0$\\
\bottomrule
\end{tabular}

\begin{tabular}{c|cccc}
\toprule
Binary& Break& Switch& Track& Arrive\\
\midrule
$100$& $0$& $0$& $0$& $1$\\
$001$& $0$& $1$& $1$& $1$\\
$001$& $1$& $0$& $2$& $0$\\
$001$& $1$& $1$& $2$& $0$\\
\bottomrule
\end{tabular}

\begin{tabular}{c|cccc}
\toprule
Binary& Break& Switch& Track& Arrive\\
\midrule
$001$& $0$& $0$& $0$& $1$\\
$100$& $0$& $1$& $1$& $1$\\
$100$& $1$& $0$& $2$& $0$\\
$100$& $1$& $1$& $2$& $0$\\
\bottomrule
\end{tabular}

\begin{tabular}{c|cccc}
\toprule
Binary& Break& Switch& Track& Arrive\\
\midrule
$100$& $0$& $0$& $0$& $1$\\
$100$& $0$& $1$& $1$& $1$\\
$001$& $1$& $0$& $2$& $0$\\
$001$& $1$& $1$& $2$& $0$\\
\bottomrule
\end{tabular}

\begin{tabular}{c|cccc}
\toprule
Binary& Break& Switch& Track& Arrive\\
\midrule
$001$& $0$& $0$& $0$& $1$\\
$001$& $0$& $1$& $1$& $1$\\
$100$& $1$& $0$& $2$& $0$\\
$100$& $1$& $1$& $2$& $0$\\
\bottomrule
\end{tabular}

\begin{tabular}{c|cccc}
\toprule
Binary& Break& Switch& Track& Arrive\\
\midrule
$100$& $0$& $1$& $1$& $1$\\
$001$& $0$& $0$& $0$& $1$\\
$001$& $1$& $0$& $2$& $0$\\
$001$& $1$& $1$& $2$& $0$\\
\bottomrule
\end{tabular}

\begin{tabular}{c|cccc}
\toprule
Binary& Break& Switch& Track& Arrive\\
\midrule
$001$& $0$& $1$& $1$& $1$\\
$100$& $0$& $0$& $0$& $1$\\
$100$& $1$& $0$& $2$& $0$\\
$100$& $1$& $1$& $2$& $0$\\
\bottomrule
\end{tabular}

\begin{tabular}{c|cccc}
\toprule
Binary& Break& Switch& Track& Arrive\\
\midrule
$100$& $0$& $0$& $0$& $1$\\
$100$& $0$& $1$& $1$& $1$\\
$100$& $1$& $0$& $2$& $0$\\
$001$& $1$& $1$& $2$& $0$\\
\bottomrule
\end{tabular}

\begin{tabular}{c|cccc}
\toprule
Binary& Break& Switch& Track& Arrive\\
\midrule
$001$& $0$& $0$& $0$& $1$\\
$001$& $0$& $1$& $1$& $1$\\
$001$& $1$& $0$& $2$& $0$\\
$100$& $1$& $1$& $2$& $0$\\
\bottomrule
\end{tabular}

\begin{tabular}{c|cccc}
\toprule
Binary& Break& Switch& Track& Arrive\\
\midrule
$100$& $0$& $1$& $1$& $1$\\
$100$& $1$& $0$& $2$& $0$\\
$001$& $0$& $0$& $0$& $1$\\
$001$& $1$& $1$& $2$& $0$\\
\bottomrule
\end{tabular}

\begin{tabular}{c|cccc}
\toprule
Binary& Break& Switch& Track& Arrive\\
\midrule
$001$& $0$& $1$& $1$& $1$\\
$001$& $1$& $0$& $2$& $0$\\
$100$& $0$& $0$& $0$& $1$\\
$100$& $1$& $1$& $2$& $0$\\
\bottomrule
\end{tabular}

\begin{tabular}{c|cccc}
\toprule
Binary& Break& Switch& Track& Arrive\\
\midrule
$100$& $0$& $0$& $0$& $1$\\
$100$& $1$& $0$& $2$& $0$\\
$001$& $0$& $1$& $1$& $1$\\
$001$& $1$& $1$& $2$& $0$\\
\bottomrule
\end{tabular}

\begin{tabular}{c|cccc}
\toprule
Binary& Break& Switch& Track& Arrive\\
\midrule
$001$& $0$& $0$& $0$& $1$\\
$001$& $1$& $0$& $2$& $0$\\
$100$& $0$& $1$& $1$& $1$\\
$100$& $1$& $1$& $2$& $0$\\
\bottomrule
\end{tabular}

\begin{tabular}{c|cccc}
\toprule
Binary& Break& Switch& Track& Arrive\\
\midrule
$100$& $1$& $0$& $2$& $0$\\
$001$& $0$& $0$& $0$& $1$\\
$001$& $0$& $1$& $1$& $1$\\
$001$& $1$& $1$& $2$& $0$\\
\bottomrule
\end{tabular}

\begin{tabular}{c|cccc}
\toprule
Binary& Break& Switch& Track& Arrive\\
\midrule
$001$& $1$& $0$& $2$& $0$\\
$100$& $0$& $0$& $0$& $1$\\
$100$& $0$& $1$& $1$& $1$\\
$100$& $1$& $1$& $2$& $0$\\
\bottomrule
\end{tabular}

\end{multicols}
\def\thesection{\Alph{section}}

\end{document}